\documentclass{article}
\usepackage{pgfplots}
\usepackage{pgfplotstable}
\pgfplotsset{compat=1.18}
\pgfplotsset{
    discard if not/.style 2 args={
        x filter/.append code={
            \edef\tempa{\thisrow{#1}}
            \edef\tempb{#2}
            \ifx\tempa\tempb
            \else
                
            \fi
        }
    }
}

\definecolor{icmlgreen}{HTML}{009E73}
\definecolor{icmlblue}{HTML}{0072B2}
\definecolor{icmlmagenta}{HTML}{CC79A7}
\definecolor{icmlorange}{HTML}{D55E00}

\usepackage{microtype}
\usepackage{graphicx}
\usepackage{caption}
\usepackage{subcaption}
\usepackage{booktabs} 
\usepackage{float}
\usepackage{cuted}   
\usepackage{capt-of} 
\usepackage{multirow}

\usepackage{hyperref}
\usepackage{enumitem}
\usepackage{comment}

\usepackage[preprint]{icml2026}

\usepackage{amsmath}
\usepackage{amssymb}
\usepackage{mathtools}
\usepackage{amsthm}

\usepackage[capitalize,noabbrev]{cleveref}

\theoremstyle{plain}
\newtheorem{theorem}{Theorem}[section]
\newtheorem{proposition}[theorem]{Proposition}
\newtheorem{lemma}[theorem]{Lemma}
\newtheorem{corollary}[theorem]{Corollary}
\theoremstyle{definition}
\newtheorem{definition}[theorem]{Definition}
\newtheorem{assumption}[theorem]{Assumption}
\theoremstyle{remark}
\newtheorem{remark}[theorem]{Remark}

\theoremstyle{plain}
\newcommand{\newcustomtheorem}[2]{
  \newenvironment{#1}[1]{
    \renewcommand{\thetheorem}{\ref{##1}}
    \begin{theorem}}{\end{theorem}}
}
\newcustomtheorem{restatedtheorem}{thm:restated}

\newcommand{\newcustomproposition}[2]{
  \newenvironment{#1}[1]{
    \renewcommand{\thetheorem}{\ref{##1}}
    \begin{proposition}}{\end{proposition}}
}
\newcustomproposition{restatedproposition}{proposition:restated}

\newcustomproposition{restatedlemma}{lemma:restated}

\newcustomproposition{restatedcorollary}{corollary:restated}

\theoremstyle{definition}

\newcustomproposition{restateddefinition}{definition:restated}

\newcustomproposition{restatedassumption}{assumption:restated}

\theoremstyle{remark}

\newcustomproposition{restatedremark}{remark:restated}

\newcommand{\citepos}[1]{\citeauthor{#1}'s (\citeyear{#1})}

\DeclareMathOperator{\arcsinh}{arcsinh}

\usepackage[textsize=small]{todonotes}

\icmltitlerunning{Fast and Geometrically Grounded Lorentz Neural Networks}

\begin{document}

\twocolumn[
\icmltitle{Fast and Geometrically Grounded Lorentz Neural Networks}



\icmlsetsymbol{equal}{*}

\begin{icmlauthorlist}
\icmlauthor{Robert van der Klis}{equal,eth}
\icmlauthor{Ricardo Chávez Torres}{equal}
\icmlauthor{Max van Spengler}{uva}
\icmlauthor{Yuhui Ding}{eth}
\icmlauthor{Thomas Hofmann}{eth}
\icmlauthor{Pascal Mettes}{uva}
\end{icmlauthorlist}

\icmlaffiliation{uva}{University of Amsterdam}
\icmlaffiliation{eth}{Department of Computer Science, ETH Zurich}

\icmlcorrespondingauthor{Robert van der Klis}{robert.vanderklis@inf.ethz.ch}
\icmlcorrespondingauthor{Ricardo Chávez Torres}{chaveztorresricardo@gmail.com}

\icmlkeywords{Machine Learning, ICML}

\vskip 0.3in
]




\printAffiliationsAndNotice{\icmlEqualContribution} 

\begin{strip}
    \par\vspace{-7em} 
    \centering
    \captionsetup{type=figure}
    \begin{subfigure}[b]{0.35\linewidth}
        \centering
        \includegraphics[width=\linewidth]{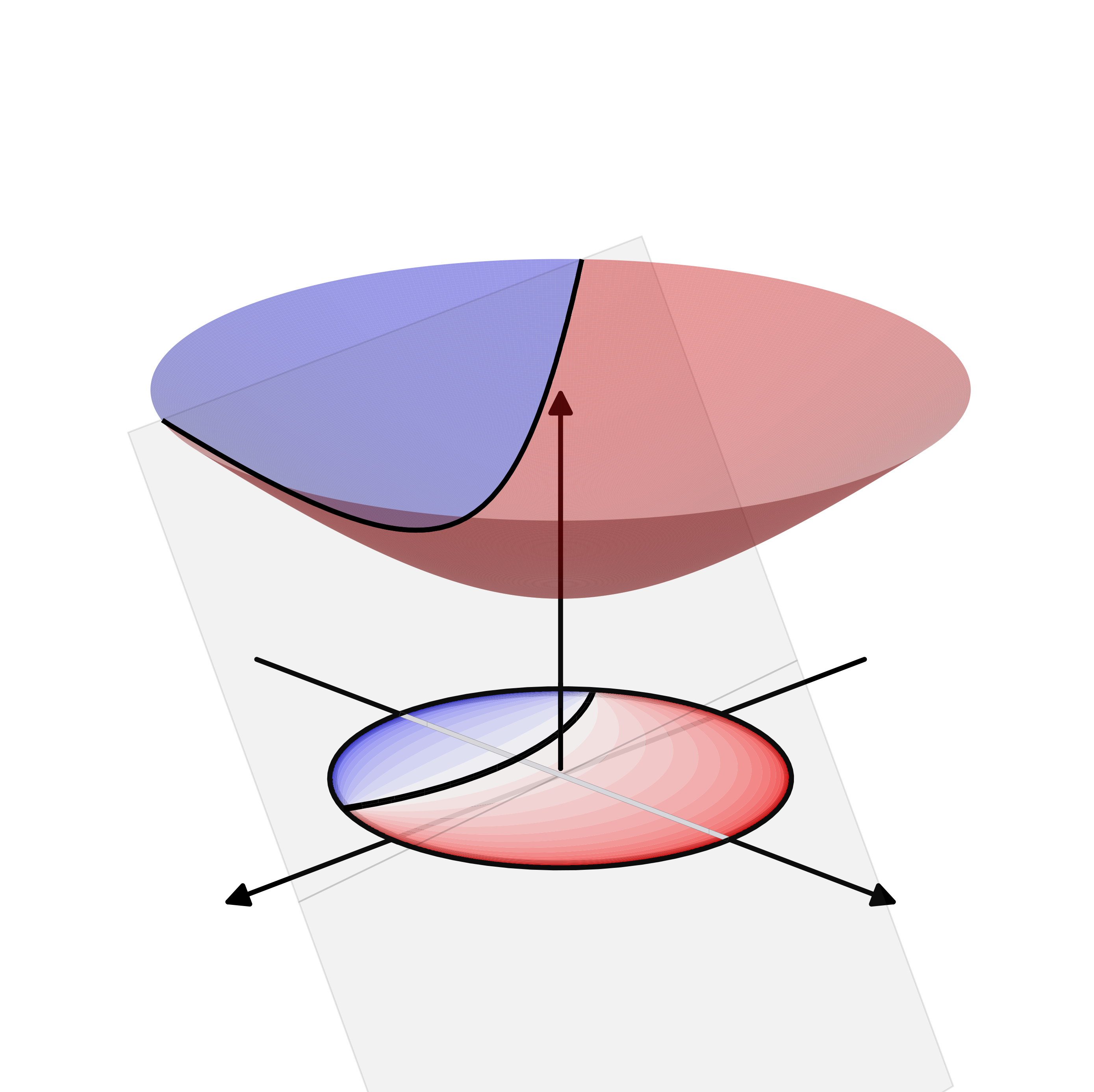}
        \caption{Our Lorentz linear layer.}
        \label{fig:sub3}
    \end{subfigure}
    \hspace{10pt}
    \begin{subfigure}[b]{0.35\linewidth}
        \begin{subfigure}[b]{0.48\linewidth}
            \centering  \includegraphics[width=\linewidth]{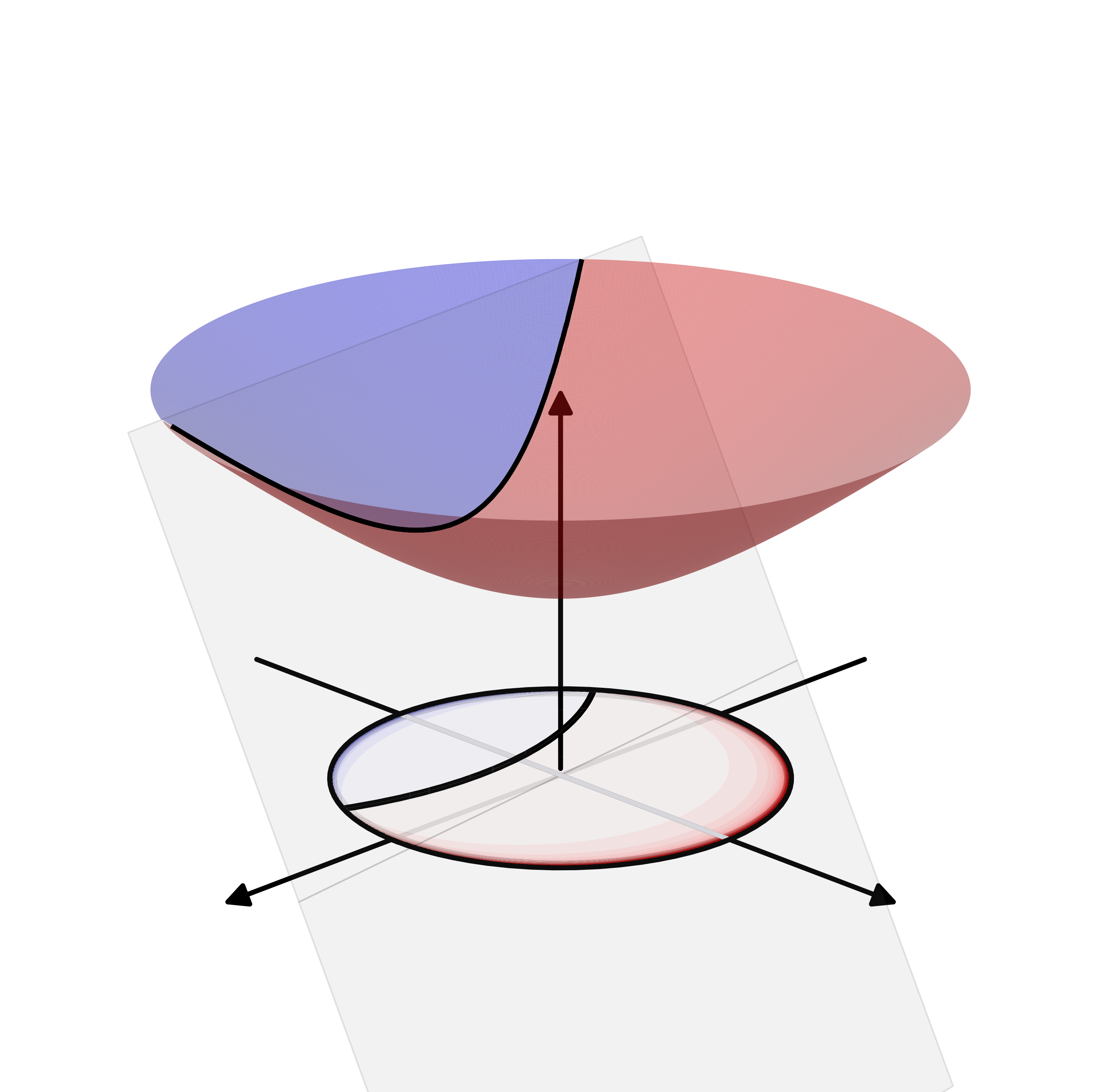}
        \end{subfigure}
        \hfill
        \begin{subfigure}[b]{0.48\linewidth}
            \centering            \includegraphics[width=\linewidth]{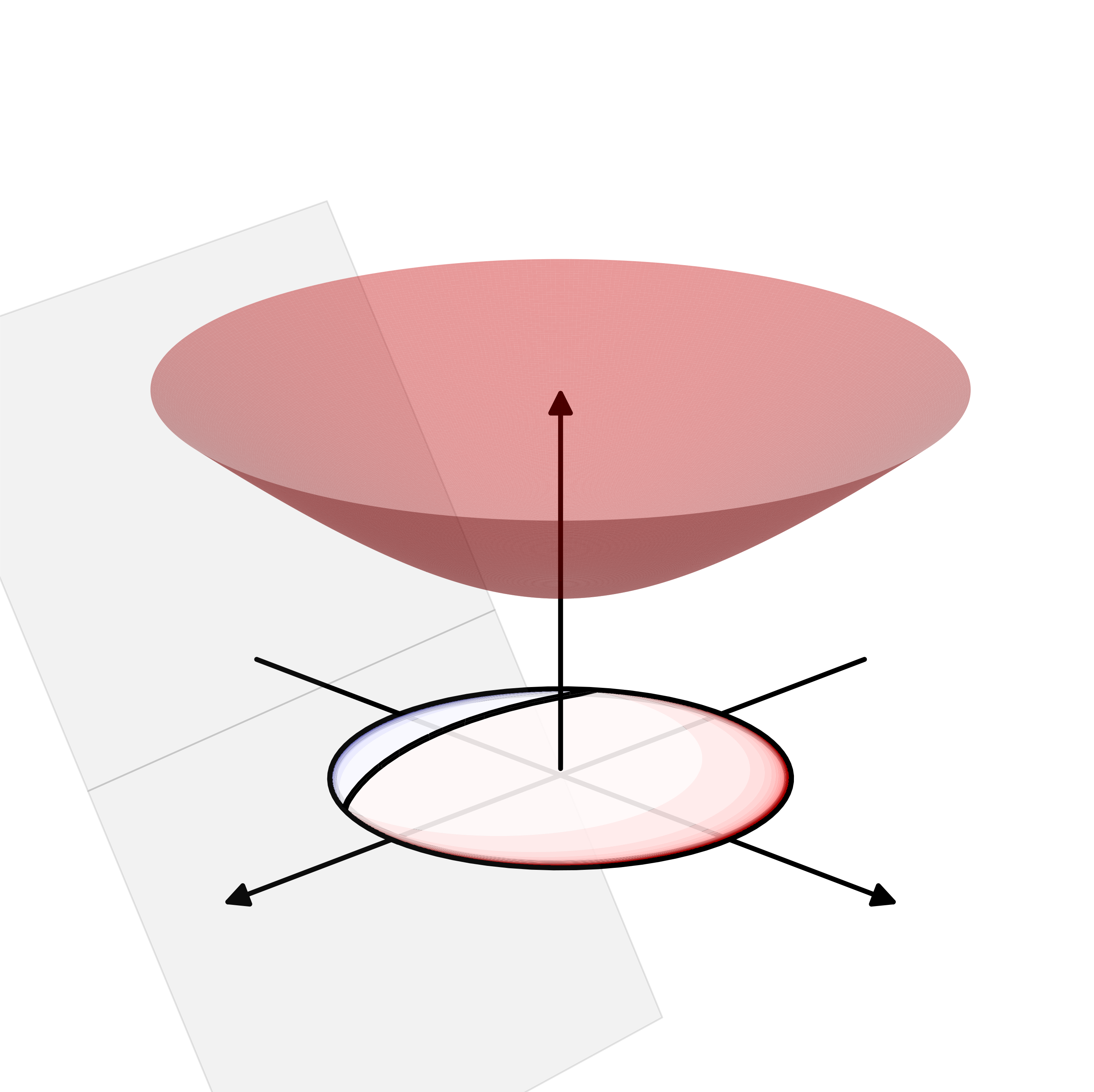}
        \end{subfigure}
        \\
        \begin{subfigure}[b]{0.48\linewidth}
            \centering            \includegraphics[width=\linewidth]{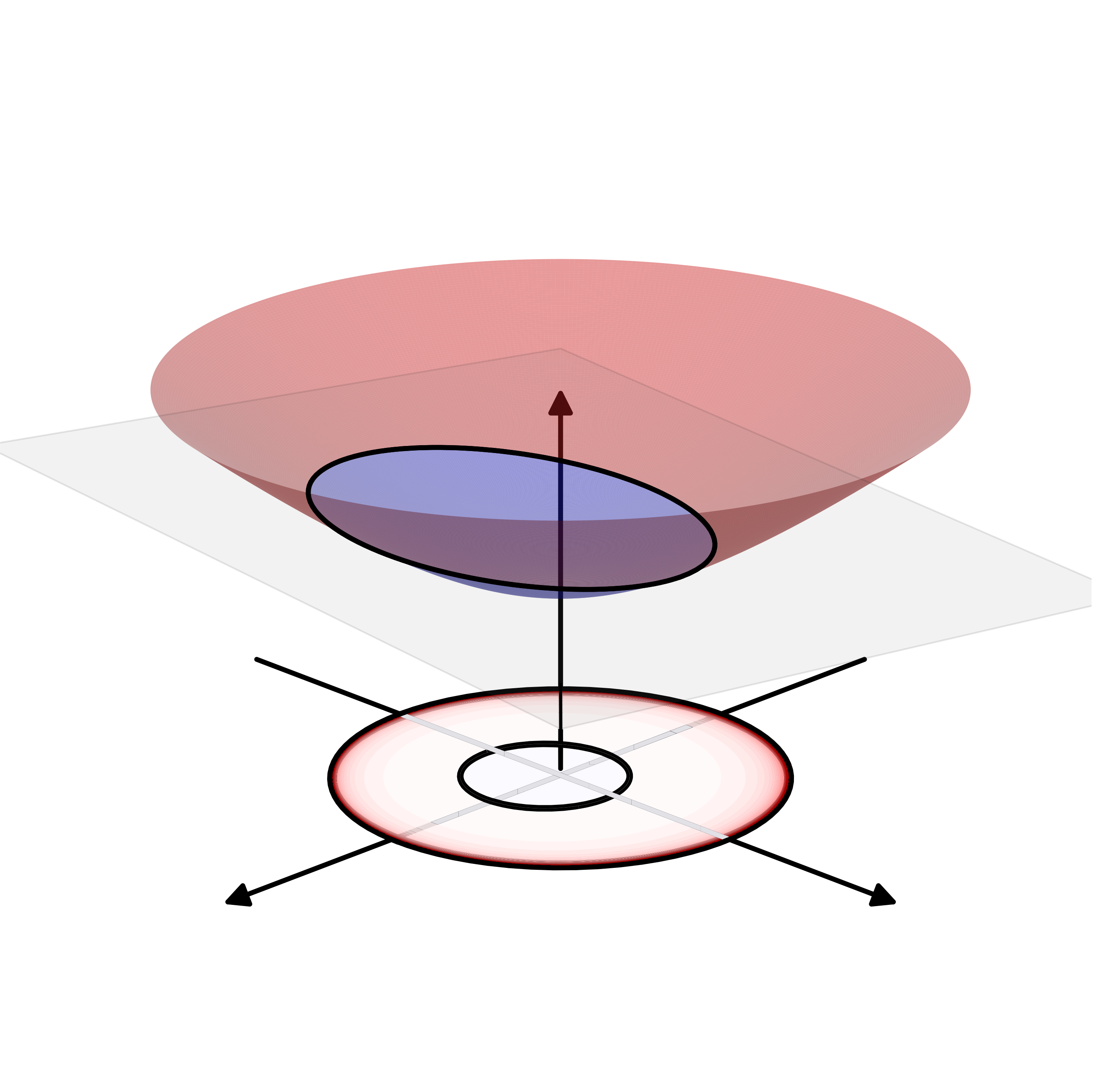}
        \end{subfigure}
        \hfill
        \begin{subfigure}[b]{0.48\linewidth}
            \centering            \includegraphics[width=\linewidth]{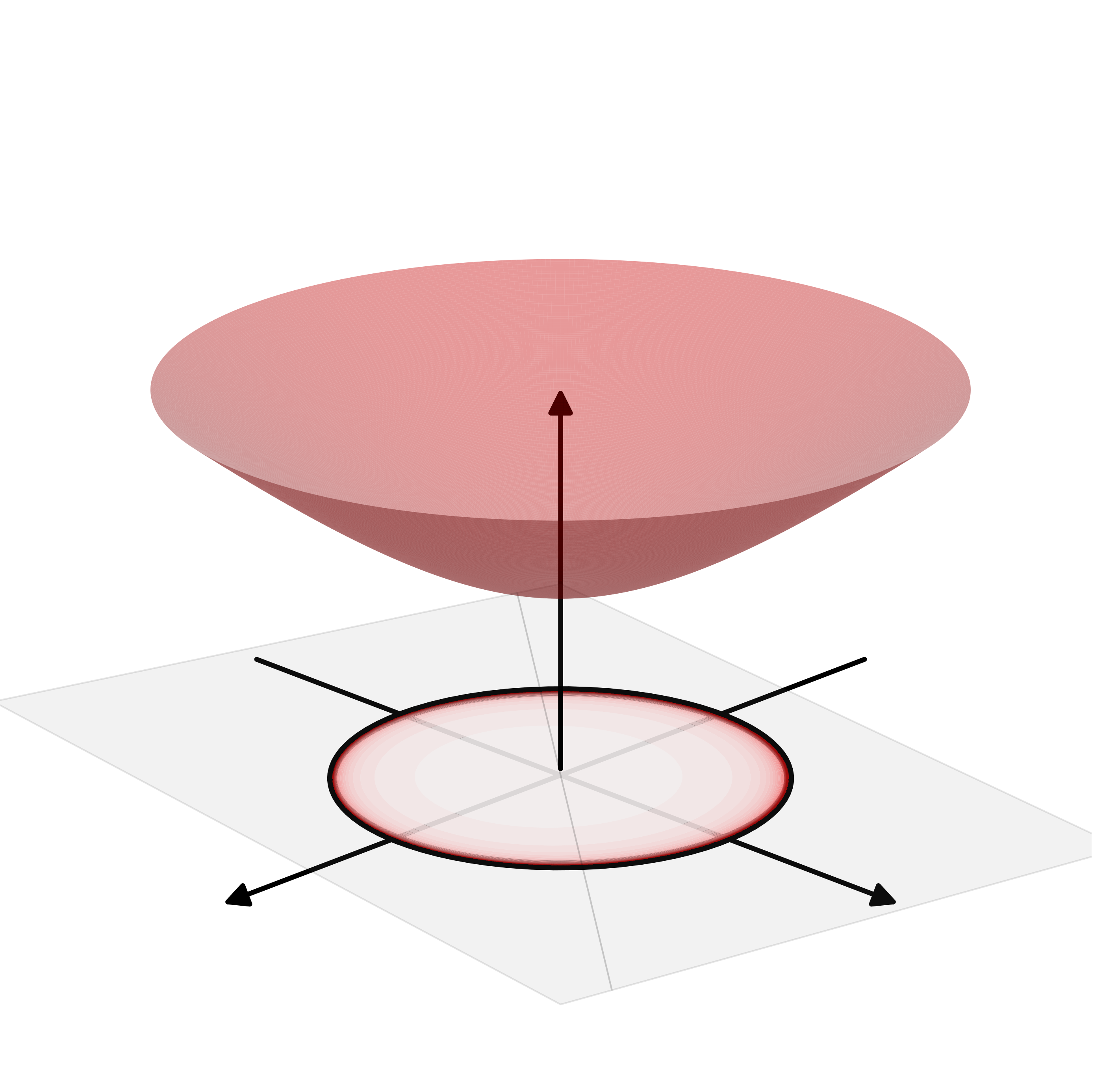}
        \end{subfigure}
        \caption{\citepos{chenFullyHyperbolicNeural2022} Lorentz linear layer.}
        \label{fig:sub4}
    \end{subfigure}
    \par\vspace{0em} 
    \caption{
    \textbf{Visualizing decision boundaries in Lorentz model, with level curves projected on the Poincaré disk.}
    The black line denotes a class boundary.
    The blue and red colors denote positive and negative distances to the decision boundary. \textbf{(a)}: We propose a formulation of a linear layer in the Lorentz model which ensures that the decision boundary forms a proper hyperbolic hyperplane. \textbf{(b)}: The current leading formulation of the Lorentzian linear layer, where intersections do not form valid hyperbolic hyperplanes in general, leading to unwanted configurations such as in the upper right and lower examples, as well as a lack of exponential growth of distances.
    }
    \label{fig:three_plots}
\end{strip}

\begin{abstract}
Hyperbolic space is quickly gaining traction as a promising geometry for hierarchical and robust representation learning. A core open challenge is the development of a mathematical formulation of hyperbolic neural networks that is both efficient and captures the key properties of hyperbolic space. The Lorentz model of hyperbolic space has been shown to enable both fast forward and backward propagation. However, we prove that, with the current formulation of Lorentz linear layers, the hyperbolic norms of the outputs scale logarithmically with the number of gradient descent steps, nullifying the key advantage of hyperbolic geometry. We propose a new Lorentz linear layer grounded in the well-known ``distance-to-hyperplane" formulation. We prove that our formulation results in the usual linear scaling of output hyperbolic norms with respect to the number of gradient descent steps. Our new formulation, together with further algorithmic efficiencies through Lorentzian activation functions and a new caching strategy results in neural networks fully abiding by hyperbolic geometry while simultaneously bridging the computation gap to Euclidean neural networks. Code available at: \url{https://github.com/robertdvdk/hyperbolic-fully-connected}.
\end{abstract}

\section{Introduction}
Hyperbolic geometry has garnered significant interest in recent years in deep learning. Amongst others, hyperbolic manifolds hold promise for embedding hierarchical structures with arbitrarily low distortion \citep{sarkar2012}.
Hierarchical structures are ubiquitous: from natural language to computer vision and graph tasks, all either have inherent hierarchical organizations or benefit from hierarchical knowledge. In semantics, the relations between word semantics are hierarchical, which is important for word embeddings \citep{nickelPoincareEmbeddingsLearning2017, tifrea2018pPoincareGloveHyperbolic2018} and for language models \citep{patil2025hyperboliclargelanguagemodels, he2025helm, yang2025hyperbolic}. In computer vision, hyperbolic learning improves hierarchical visual recognition \citep{khrulkovHyperbolicImageEmbeddings2020, Kwon_2024_CVPR, Wang_2025_ICCV}, robustness \citep{vanSpengler2025AdversarialAttacks, Bi_Yi_Zhan_Ji_Xia_2025}, few-shot learning \citep{Liu_2020_CVPR, Moreira_2024_WACV, HAMZAOUI2024151}, and more \citep{Mettes2024Survey, He2025surveyHyperbolicFoundationModels}.
In graph tasks, ``hierarchy is a central organizing principle of complex networks" \citep{clausetHierarchicalStructurePrediction2008}, with hyperbolic graph networks as powerful and broadly applicable tools \citep{chamiHyperbolicGraphConvolutional2019a, NEURIPS2019_103303dd, NEURIPS2023_8b6a8b01}.

However, building fully hyperbolic neural networks is not straightforward: much of deep learning relies on the definition of the fully connected layer, which serves as the basic building block for more complex components such as convolutions, attention mechanisms, or graph message-passing operations. In past years, seminal works by \citet{ganeaHyperbolicNeuralNetworks2018} and \citet{shimizuHyperbolicNeuralNetworks2020} formulated a hyperbolic linear layer in the Poincaré model. While theoretically sound, operating in the Poincaré model is computationally inefficient \citep{chenFullyHyperbolicNeural2022, pmlr-v202-mishne23a}. As a consequence, several recent works started operating in the Lorentz model of hyperbolic space \citep{chenFullyHyperbolicNeural2022, bdeir2023FullyVision, yangHypformerExploringEfficient2024a} as a computationally much more efficient alternative.

In this work, however, we find that the current formulation of hyperbolic layers in the Lorentz model is inconsistent with the geometric interpretation of Euclidean layers, interestingly resulting in a loss of precisely the efficient hierarchical embedding property of hyperbolic geometry that we want to encapsulate in our neural networks, see Figure \ref{fig:three_plots}.
This motivates us to propose a new formulation of the foundational building block of deep learning in the Lorentz model of hyperbolic space.

We propose a new Lorentz linear layer. We strive to get the best out of the theoretical soundness of the Poincaré model formulation and the computational efficiency of the Lorentz model. We make the following contributions:
\begin{enumerate}[topsep=0pt, itemsep=-1ex]
    \item \textbf{Proof of the pathology:} We prove that the hyperbolic norms of embeddings generated by current formulations of Lorentz linear layers scale logarithmically rather than linearly with the number of gradient steps. We demonstrate this behaviour on a toy experiment.
    \item \textbf{Geometric unification:} We formulate a new Lorentz linear layer based on the ``distance-to-hyperplane" formulation used in Poincaré models, and we prove that this new layer resolves the pathology.
    \item \textbf{Algorithmic efficiency:} We introduce Lorentzian activations, a family of activation functions that simplify the computation of our layer. We also introduce a way to cache our parameters, yielding inference speeds of our Lorentz linear layer that are 2.9x as fast as the current state-of-the-art Lorentz neural networks, and 8.3x as fast as current state-of-the-art Poincaré neural networks.
    \item \textbf{FGG-LNN:} We use all these building blocks to create Fast and Geometrically Grounded Lorentz Neural Networks (FGG-LNN). We show that this new formulation, together with a well-optimized implementation, speeds up training times of a ResNet-18 over state-of-the-art Lorentz neural networks by 3.5x, and over state-of-the-art Poincaré neural networks by 7.5x, closing the gap towards Euclidean networks.
\end{enumerate}

\section{Preliminaries}
A Riemannian manifold is a tuple $(M, g)$, where $M$ is a smooth manifold and $g$ is a Riemannian metric, defining a smooth inner product $g_x(\cdot, \cdot) = \langle \cdot, \cdot \rangle_x$ on the tangent space $T_x M$ at each point $x \in M$. The metric induces a distance function $d_M(x, y)$, which is the length of the shortest path (geodesic) connecting $x, y \in M$.

We use the Lorentz model (or hyperboloid model) of hyperbolic geometry. Let $\mathbb{R}^{D+1}$ be the ambient space (Minkowski space). For any vector $\mathbf{x} \in \mathbb{R}^{D+1}$, we separate the time component and the space components and denote them as $\mathbf{x} = (x_1, \overline{\mathbf{x}})$, where $x_1 \in \mathbb{R}$ and $\overline{\mathbf{x}} \in \mathbb{R}^D$.
This space is equipped with the Lorentz inner product denoted $\circ$, with Minkowski metric $\mathbf{I}_{1,D} = \text{diag}(-1, 1, \dots, 1)$:
\begin{equation}
    \mathbf{x} \circ \mathbf{y} = -x_1 y_1 + \overline{\mathbf{x}}^\top \overline{\mathbf{y}} = \mathbf{x}^\top \mathbf{I}_{1,D}\mathbf{y}.
\end{equation}

The hyperbolic space $\mathbb{L}_\kappa^D$ with constant negative curvature $-\kappa$ ($\kappa > 0$) is defined as the upper sheet of a hyperboloid:
\begin{equation}
    \mathbb{L}_\kappa^D = \left\{ \mathbf{z} \in \mathbb{R}^{D+1} \mid \mathbf{z} \circ \mathbf{z} = -1/\kappa, \, z_1 > 0 \right\}.
\end{equation}
The geodesic distance between any two points in $\mathbb{L}_\kappa^D$ is: 
\begin{equation}
    \label{eqn:dist}
    d(\mathbf{x}, \mathbf{y}) = \frac{1}{\sqrt{\kappa}}\text{arccosh}(-\kappa (\mathbf{x} \circ \mathbf{y})).
\end{equation}
The tangent space $T_{\mathbf{z}} \mathbb{L}_\kappa^D$ at a point $\mathbf{z}$ is the set of vectors Minkowski-orthogonal to $\mathbf{z}$:
\begin{equation}
    T_{\mathbf{z}} \mathbb{L}_\kappa^D = \{ \mathbf{v} \in \mathbb{R}^{D+1} \mid \mathbf{z} \circ \mathbf{v} = 0 \}.
\end{equation}
We define the \textit{origin} of the manifold as $\mathbf{o} = (\frac{1}{\sqrt{\kappa}}, 0, \dots, 0)$.
For learning on manifolds, we require maps to transition between the manifold and its tangent spaces, as well as a way to move vectors across the manifold, specifically the exponential map from tangent space to manifold ($\text{Exp}_{\mathbf{x}}: T_{\mathbf{x}} \mathbb{L}_\kappa^D \to \mathbb{L}_\kappa^D$), the logarithmic map for the reverse ($\text{Log}_{\mathbf{x}}: \mathbb{L}_\kappa^D \to T_{\mathbf{x}} \mathbb{L}_\kappa^D$), and the parallel transport which transports a tangent vector $\mathbf{v}$ along the geodesic from $\mathbf{x}$ to $\mathbf{y}$ while preserving the inner product ($\Gamma_{\mathbf{x} \to \mathbf{y}}: T_{\mathbf{x}} \mathbb{L}_\kappa^D \to T_{\mathbf{y}} \mathbb{L}_\kappa^D$). We refer to \Cref{appendix:hyperbolic_ops} for full definitions.

There are several different (pseudo)norms we will use in this paper. The first is the Euclidean norm: we will denote this by $\|\mathbf{v}\|_E^2 = \mathbf{v} \cdot \mathbf{v}$. Analogous to the Euclidean case, we define the pseudonorm $\|\mathbf{v}\|_\mathcal{L}^2 = \mathbf{v} \circ \mathbf{v}$ (not positive definite). When $\mathbf{v} \in T_\mathbf{z} \mathbb{L}_\kappa^D$, we write $\|\mathbf{v}\|_\mathbf{z}^2 = \mathbf{v} \circ \mathbf{v}$, which is a proper norm when restricted to the tangent space. 

Finally, we define the \textit{hyperbolic norm} of a point $\mathbf{z} \in \mathbb{L}_\kappa^D$ as its 
geodesic distance to the origin:
\begin{equation}
    \|\mathbf{z}\|_\mathcal{H} = d(\mathbf{z}, \mathbf{o}) = \frac{1}{\sqrt{\kappa}}\text{arccosh}(\sqrt{\kappa} \, z_1).
\end{equation}

\section{Pathology in Current Lorentz Linear Layers}
\label{sec:pathology}
Current Lorentz linear layers follow \citet{chenFullyHyperbolicNeural2022} and calculate $\mathbf{Wx} \in \mathbb{R}^{D_{\text{out}}}$ as in the Euclidean setting, and treat this as the spatial component of the output, such that $\overline{\mathbf{y}} = \mathbf{Wx}$. To ensure that this point lies on the hyperboloid, the time coordinate $y_1$ is then defined as
\begin{equation}
\label{eqn:time_component}
    y_1 = \sqrt{\|\mathbf{Wx}\|_E^2 + \frac{1}{\kappa}}.
\end{equation}

A fundamental problem with this approach is that the weights of the layer are (approximately) linearly related\footnote{By ``approximate" we mean that the weights are linearly related to the spatial component of the output, and the time component is related to the spatial component as in \Cref{eqn:time_component}, which is asymptotically linear.} to the Euclidean norms of outputs generated by that layer.

In Euclidean space, this is desirable, but in the Lorentz model, it is not, as distances in the Lorentz model grow logarithmically with the Euclidean norm of the coordinates (\Cref{eqn:dist}). This means that, in order to embed a hierarchy of depth $h$, the Euclidean norm of the embeddings needs to be on the order of $e^h$. Given the approximate linear relation between the weights and the output norms, it follows that the Frobenius norm of the weight matrix must be on the order of $e^h$ as well. With bounded updates, this requires a number of gradient descent steps that is proportional to $e^h$.

We will now formalise the above intuition. First, we will show that the increase in hyperbolic norm achievable by \citepos{chenFullyHyperbolicNeural2022}  definition grows at most logarithmically with the number of gradient descent steps. Second, we will show that it is not possible to simply shrink the distances between nodes in the hierarchy to solve the problem. Finally, we will show that this implies that, with \citepos{chenFullyHyperbolicNeural2022} definition, the number of gradient updates required to embed a hierarchy of depth $h$ with low distortion (i.e. approximately preserving pairwise graph distances) grows exponentially with $h$.

\begin{proposition}
    \label{prop:upperbound}
    Let $\mathbf{W}_n$ be a weight matrix after $n$ steps of optimization with updates bounded in Frobenius norm by $\delta$ (i.e., $\|\Delta \mathbf{W}_t\|_F \leq \delta$). For any input $\mathbf{x} \in \mathbb{L}_\kappa^{D_\text{in}}$ and output $\mathbf{y} \in \mathbb{L}_\kappa^{D_{\text{out}}}$, the amount by which $\mathbf{W}_n$ can increase the distance to the origin is upper bounded by $\mathcal{O}(\ln n)$: we have
    $$d(\mathbf{o}, \mathbf{y}) - d(\mathbf{o}, \mathbf{x}) = \mathcal{O}(\ln n).$$
\end{proposition}

\begin{proof}
    See \Cref{appendix:proofs}.
\end{proof}
The upper bound that we assume in the proposition is easily satisfied: gradient clipping, for example, satisfies the assumption. However, even without gradient clipping, this condition typically holds in stable training regimes: standard optimization practices, such as the use of small learning rates and weight decay, inherently prevent the norm of the updates from growing without bound.

Now, we show why the above is a problem: namely, we prove a lower bound on the distance between nodes and their children for low-distortion tree embeddings: 
\begin{proposition}
    \label{prop:treedistance}
    Embedding an $m$-ary tree with low distortion in $\mathbb{L}_\kappa^{D}$ requires a minimum distance between nodes and their children of
    $$s = \Omega\left(\frac{\ln m}{(D_{\text{in}}-1)\sqrt{\kappa}}\right).$$
\end{proposition}
\begin{proof}
    See \Cref{appendix:proofs}.
\end{proof}

Finally, we connect the two: the upper bound on the increase in hyperbolic norm of \citepos{chenFullyHyperbolicNeural2022} layer, combined with the lower bound on the distances between nodes and their children, results in exponential growth in the required number of steps to embed a tree.
\begin{proposition}
    \label{cor:chen}
    The linear layer formulation by \citet{chenFullyHyperbolicNeural2022} requires $\Omega(e^h)$ gradient updates to embed a hierarchy of depth $h$ with low distortion.
\end{proposition}
\begin{proof}
    See \Cref{appendix:proofs}.
\end{proof}

Naturally, such exponential scaling is undesirable since it prevents the layer from actually embedding hierarchies. We therefore want to develop a formulation that appropriately generalizes the linear scaling from Euclidean linear layers to hyperbolic space, allowing the layer to leverage the curvature of the space and, consequently, effectively represent hierarchical structures.

\section{Methodology}
To formulate a Lorentz linear layer that does not suffer from the pathology above, we will follow the procedure of \citet{ganeaHyperbolicNeuralNetworks2018, shimizuHyperbolicNeuralNetworks2020}: that is, we first generalise Euclidean multinomial logistic regression (MLR) to the Lorentz model of hyperbolic space, after which we will use this Lorentz MLR to formulate a Lorentz linear layer. While \citet{bdeir2023FullyVision} precede us in formulating a generalisation of MLR to the Lorentz model, they did not subsequently use this to define a linear layer. We will also make several different choices in our derivation of Lorentz MLR.

\subsection{Deriving the Lorentz Linear Layer}
We recall that a standard Euclidean linear layer $f: \mathbb{R}^{D_{\text{in}}} \to \mathbb{R}^{D_{\text{out}}}$ is defined by $\mathbf{z} = \mathbf{W}\mathbf{x} + \mathbf{b},$ where the parameters consist of the weight matrix $\mathbf{W} \in \mathbb{R}^{D_{\text{out}} \times D_{\text{in}}}$, and the bias vector $\mathbf{b} \in \mathbb{R}^{D_{\text{out}}}$. Denoting the $i$-th row of $\mathbf{W}$ as $\mathbf{w}^{(i)}$, the pre-activations $\mathbf{z}$ are given by $z_i = {\mathbf{w}^{(i)}}^\top \mathbf{x} + b_i$. The insight given by \citet{ganeaHyperbolicNeuralNetworks2018, shimizuHyperbolicNeuralNetworks2020} is that we can view this as the (signed and scaled)\footnote{The sign depends on the side of the hyperplane on which $\mathbf{x}$ lies, and the scale is the magnitude of the weights.} distance from the input $\mathbf{x}$ to the hyperplane $H_{\mathbf{w}^{(i)},b_i}= \{\mathbf{u} \in \mathbb{R}^{D_{in}} \ | \ \mathbf{w}^{(i)\top} \mathbf{u} + b_i = 0\}$. In a neural network, these pre-activations $\mathbf{z}$ are then typically followed by an element-wise non-linearity $h$ such that $\mathbf{y} = h(\mathbf{z})$.

Thus, to rigorously generalize the linear layer to the Lorentz model, we need to define: (1) the \textbf{hyperplanes in the input manifold} and how to parametrize them with the weights and biases; (2) an expression for the \textbf{pre-activations} that includes how to compute the signed and scaled distance from the input to the hyperplanes; (3) how to apply the \textbf{activation function}; and (4) how to \textbf{compute the output} in the target manifold from the activations. We will now cover each in turn.

\subsection*{Step 1: Hyperplanes in the Input Manifold}

In Euclidean space, a hyperplane is the set of all vectors $\mathbf{x}$ such that $(\mathbf{x} - \mathbf{p}) \cdot \mathbf{n} = 0$, where $\mathbf{p}$ is any point in the hyperplane. We call $\mathbf{p}$ the \textit{reference point}. So now, let us generalise this definition to the Lorentz manifold. We write $M_{\text{in}}$ for the input manifold. For a visual illustration of the geometric construction of these hyperplanes and the role of the parameters, see \cref{appendix:geometric_interpretation} and \cref{fig:3D_2D_LNN}.

We start by parametrising the reference point: we define it as the point that is displaced from the origin by $-\frac{b_i}{\|\mathbf{w}^{(i)}\|_\mathbf{o}}$ in the direction of the weights \citep{bishop2006pattern}:
\begin{equation}
    \mathbf{p}_i = \text{Exp}_{\mathbf{o}} \left( - \frac{b_i}{\|\mathbf{w}^{(i)}\|_{\mathbf{o}}} \frac{\mathbf{w}^{(i)}}{\|\mathbf{w}^{(i)}\|_{\mathbf{o}}} \right).
\end{equation}
Then, the hyperplane is the set of all points $\mathbf{q}$ in the hyperplane such that the geodesic from $\mathbf{p}_i$ to $\mathbf{q}$ is orthogonal to the transported weight vector $\Gamma_{\mathbf{o}\to\mathbf{p}_i} (\mathbf{w}^{(i)})$: 
\begin{equation}
    \label{eq:general_hyperplane}
    H_{\mathbf{w}^{(i)}, b_i} = \left\{ \mathbf{q} \in M_{\text{in}} \mid \left\langle \text{Log}_{\mathbf{p}_i}(\mathbf{q}), \Gamma_{\mathbf{o} \to \mathbf{p}_i}(\mathbf{w}^{(i)}) \right\rangle_{\mathbf{p}_i} = 0 \right\}.
\end{equation}
Then, following the exponential map definition in \cref{eq:expmap}, the expression for $\mathbf{p}_i$ in the Lorentz linear layer is:
\begin{align}
    \label{eq:defn_p}
    \mathbf{p}_i &= \text{Exp}_\mathbf{o} \left( - \frac{b_i}{\|\mathbf{w}^{(i)}\|_\mathcal{L}}  \frac{\mathbf{w}^{(i)}}{\|\mathbf{w}^{(i)}\|_{\mathcal{L}}} \right) \\
    &= \frac{1}{\sqrt{\kappa}}\begin{pmatrix}
         \cosh\left(-\sqrt{\kappa} \frac{b_i}{\|\mathbf{w}^{(i)}\|_E}\right) \\
        \sinh\left(-\sqrt{\kappa} \frac{b_i}{\|\mathbf{w}^{(i)}\|_E} \right) \frac{\mathbf{w}^{(i)}}{\|\mathbf{w}^{(i)}\|_E}
    \end{pmatrix}.
\end{align}

In order to simplify our definition of the hyperplane (\Cref{eq:general_hyperplane}), we introduce some shorthand for the weights vector transported to the reference point that we can compute applying the parallel transport from \cref{eq:parallel_transport}:
\begin{align}
    \label{eq:defn_v}
    \mathbf{v}^{(i)} &= \Gamma_{\mathbf{o} \to \mathbf{p}_i}(\mathbf{w}^{(i)})\\
    &= \begin{pmatrix} 
    \|\mathbf{w}^{(i)}\|_E \sinh\left(-\sqrt{\kappa} \frac{b_i}{\|\mathbf{w}^{(i)}\|_E}\right) \\
    \mathbf{w}^{(i)} \cosh\left(-\sqrt{\kappa} \frac{b_i}{\|\mathbf{w}^{(i)}\|_E}\right)
    \end{pmatrix}.
\end{align}

Then, the definition in \Cref{eq:general_hyperplane} simplifies to a linear constraint in the ambient space:
\begin{proposition} \label{prop:lorentz_hyperplane_simplification}
    In the Lorentz model, the condition for the geodesic hyperplane defined in \cref{eq:general_hyperplane}, given by $\left\langle \text{Log}_{\mathbf{p}_i}(\mathbf{z}), \Gamma_{\mathbf{o} \to \mathbf{p}_i}(\mathbf{w}^{(i)}) \right\rangle_{\mathbf{p}_i} = 0$, simplifies to the ambient orthogonality condition:
    \begin{equation}
        \label{eq:lorentz_hyperplane}
        H_{\mathbf{w}^{(i)}, b_i} = \left\{ \mathbf{z} \in \mathbb{L}_{\kappa}^{D_{\text{in}}} \mid \mathbf{z} \circ \mathbf{v}^{(i)} = 0 \right\},
    \end{equation}
    where $\mathbf{v}^{(i)} = \Gamma_{\mathbf{o} \to \mathbf{p}_i}(\mathbf{w}^{(i)})$.
\end{proposition}

\begin{proof}
    See Appendix~\ref{prop:proof:lorentz_hyperplane_simplification}.
\end{proof}

So now, we have parametrised hyperplanes in the Lorentz model using a matrix $\mathbf{W}$ and a vector $\mathbf{b}$, and we have simplified the definition of hyperplanes using $\mathbf{v}^{(i)}$. Next, we will use this simplification to derive the expression for distances to hyperplanes in the next step.

\subsection*{Step 2: Pre-activations from Distances to Hyperplanes}
Rather than using distances to hyperplanes directly as pre-activations, we formulate pre-activations as signed and scaled distances to hyperplanes, similar to the Euclidean layer \citep{bishop2006pattern} and the formulation by  \citet{shimizuHyperbolicNeuralNetworks2020}. 
So for an input $\mathbf{x} \in M_{\text{in}}$, the pre-activation $z_i \in \mathbb{R}$ is the signed and scaled distance to the $i$-th hyperplane:
\begin{equation} \label{eq:signed_scaled_d_init}
    z_i = d^{\text{signed},\text{scaled}}_M(\mathbf{x}, H_{\mathbf{w^{(i)}},b_i}).
\end{equation}

The sign of the pre-activation depends on the side of the hyperplane on which the input resides: this is what makes the layer define decision boundaries. Otherwise, the neuron would not be able to determine if a point is ``inside" or ``outside" a feature boundary. To define the sign and orient the space we use the transported weight vector $\mathbf{v}^{(i)}$: the sign of the distance is the sign of the inner product of $\mathbf{v}^{(i)}$ with the logarithmic map from the reference point to the input, $\text{Log}_{\mathbf{p}_i}(\mathbf{x})$:

\begin{equation}    
\text{sign}(z_i) = \text{sign}\left( \langle \text{Log}_{\mathbf{p}_i}(\mathbf{x}), \mathbf{v}^{(i)} \rangle_{\mathbf{p}_i} \right).
\end{equation}

To mimic Euclidean linear layers, the scaling factor of the distance should be a function of the weight norm $\|\mathbf{w}^{(i)}\|_\mathbf{o}$: the weight magnitude acts as an implicitly learnable gain parameter that regulates the sensitivity of the $i$-th output dimension. Scaled distances perform better than non-scaled distances to the hyperplanes  in the Euclidean case \citep{salimansWeightNormalizationSimple2016}.

To start, we will derive the unsigned, unscaled distance to a hyperplane in the Lorentz model, after which we will add the sign and scaling factor.

\begin{theorem} \label{theorem:point_to_hyperplane_distance}
    In the Lorentz model, the hyperbolic distance from a point $\mathbf{x} \in \mathbb{L}_\kappa^{D_\text{in}}$ to the hyperplane $H_{\mathbf{w}^{(i)}, b_i} = \{ \mathbf{z} \in \mathbb{L}_\kappa^{D_\text{in}} \mid \mathbf{z} \circ \mathbf{v}^{(i)} = 0 \}$ is given by:
    \begin{equation}
        d(\mathbf{x}, H_{\mathbf{w}^{(i)}, b_i}) = \frac{1}{\sqrt{\kappa}} \text{arcsinh} \left( \sqrt{\kappa} \frac{|\mathbf{x} \circ \mathbf{v}^{(i)}|}{\|\mathbf{v}^{(i)}\|_\mathcal{L}}\right).
    \end{equation}
\end{theorem}

\begin{proof}
    See Appendix~\ref{theorem:proof:point_to_hyperplane_distance}
\end{proof}

Next, we derive the sign of our target function.

\begin{proposition}
    \label{prop:sign_equivalence}
    Let $\mathbf{p}_i \in \mathbb{L}_\kappa^{D_\text{in}}$ be the reference point of a hyperplane defined by the normal vector $\mathbf{v}^{(i)} \in T_{\mathbf{p}_i}\mathbb{L}_\kappa^{D_\text{in}}$. For any point $\mathbf{x} \in \mathbb{L}_\kappa^{D_\text{in}}$ (where $\mathbf{x} \neq \mathbf{p}_i$), the sign of the Riemannian inner product between the logarithmic map and the normal vector is equivalent to the sign of the Minkowski inner product in the ambient space:
    \begin{equation}
        \text{sign}\left( \langle \text{Log}_{\mathbf{p}_i}(\mathbf{x}), \mathbf{v}^{(i)} \rangle_{\mathbf{p}_i} \right) = \text{sign}(\mathbf{x} \circ \mathbf{v}^{(i)}).
    \end{equation}
\end{proposition}

\begin{proof}
    See Appendix~\ref{prop:proof:sign_equivalence}
\end{proof}

Combining these results leads to the signed distance formula:
\begin{equation}
    d^{\text{ signed}}(\mathbf{x}, H_{\mathbf{w}^{(i)}, b_i}) = \frac{1}{\sqrt{\kappa}} \text{arcsinh} \left( \sqrt{\kappa} \frac{\mathbf{x} \circ \mathbf{v}^{(i)}}{\|\mathbf{v}^{(i)}\|_\mathcal{L}}\right),
\end{equation}
because arcsinh preserves the sign and
\begin{equation}
    \text{sign}(\mathbf{x}\circ\mathbf{v}^{(i)})\   |\mathbf{x}\circ\mathbf{v}^{(i)}| = \mathbf{x}\circ\mathbf{v}^{(i)}.
\end{equation}
Finally, the scaling factor. One option would be to simply multiply by $\|\mathbf{w}^{(i)}\|_\mathcal{L}$ to get \begin{equation}
    \label{eqn:d1}
    d_1 = \frac{\|\mathbf{w}^{(i)}\|_\mathcal{L}}{\sqrt{\kappa}} \text{arcsinh} \left( \sqrt{\kappa} \frac{\mathbf{x} \circ \mathbf{v}^{(i)}}{\|\mathbf{v}^{(i)}\|_\mathcal{L}}\right).
\end{equation}
This is the solution proposed by \citet{ganeaHyperbolicNeuralNetworks2018, shimizuHyperbolicNeuralNetworks2020, bdeir2023FullyVision}. However, we will make a different choice: instead, we multiply by $\|\mathbf{w}^{(i)}\|_\mathcal{L}$ inside the $\text{arcsinh}$. As $\|\mathbf{w}^{(i)}\|_\mathcal{L} = \|\mathbf{v}^{(i)}\|_\mathcal{L}$ (parallel transport preserves the magnitude of tangent vectors), this allows for a simplification:
\begin{equation}
    \label{eqn:d2}
    d_2 = \frac{1}{\sqrt{\kappa}} \text{arcsinh} \left( \sqrt{\kappa} (\mathbf{x} \circ \mathbf{v}^{(i)})\right).
\end{equation}
This is a non-linear scaling of the output, that has a different geometric interpretation: while \Cref{eqn:d1} is a linear scaling of the distances by the norm of the weights vector, \Cref{eqn:d2} is a linear scaling of the sinh of the hyperbolic angle, which results in a non-linear scaling of the distance. When the curvature tends to $0$, both choices approach the same expression:
\begin{proposition} \label{prop:dist_limit_equivalence}
    In the limit where the curvature $\kappa \to 0^+$, the following two expressions for the signed distance tend to the same value:
    \begin{align*}
        d_1 &= \frac{\|\mathbf{w}^{(i)}\|_{\mathcal{L}}}{\sqrt{\kappa}} \arcsinh\left( \sqrt{\kappa} \frac{ \mathbf{x} \circ \mathbf{v}^{(i)}}{\|\mathbf{v}^{(i)}\|_{\mathcal{L}}} \right), \\
        d_2 &= \frac{1}{\sqrt{\kappa}} \arcsinh\left( \sqrt{\kappa} (\mathbf{x} \circ \mathbf{v}^{(i)}) \right).
    \end{align*}
    Specifically, both reduce to the Lorentz product $\mathbf{x} \circ \mathbf{v}^{(i)}$.
\end{proposition}
\begin{proof} 
    See Appendix~\ref{prop:proof:dist_limit_equivalence}
\end{proof}
So, we calculate the pre-activations as
\begin{align}
\label{eq:preactivations}
    z_i &= 
    d^{\text{ signed, scaled}}(\mathbf{x}, H_{\mathbf{w}^{(i)}, b_i}) \\
    &= \frac{1}{\sqrt{\kappa}} \text{arcsinh} \left( \sqrt{\kappa} \mathbf{x} \circ \mathbf{v}^{(i)}\right). 
\end{align}
Next, we will see how to apply the activation function to these pre-activations.

\subsection*{Step 3: Activation Functions}
One may question the need for activation functions in non-Euclidean geometries: prevailing wisdom dictates that activation functions introduce non-linearity, preventing a series of matrix multiplications from collapsing to a single matrix, and therefore enhancing expressivity. Given that the operations in the Lorentz model are inherently non-linear, activation functions might seem superfluous. However, \citet{vanSpengler2023PoincareResNet} found that using ReLU substantially benefited learning dynamics.

As such, we will proceed by including activation functions in our model. Given a pre-activation $z_i$, we normally compute the activation as
\begin{equation}
    \label{eq:activation}
    a_i = h(z_i),
\end{equation} where $h: \mathbb{R} \to \mathbb{R}$ is a scalar non-linearity (we write $\mathbf{h}$ for its vector equivalent). 

This works the same in the Lorentz model. However, we will define a specific class of activation functions that will be useful later.

\begin{definition}[Lorentzian activation function]
    Let $h:\mathbb{R} \to \mathbb{R}$ be a real valued function and $\kappa$ a positive real number. Then, we define the Lorentzian version of $h$ as:
    \begin{align}
        h_{\text{Lorentzian},\kappa}(x) = \frac{1}{\sqrt{\kappa}} \arcsinh(\sqrt{\kappa}h(\frac{1}{\sqrt{\kappa}}\sinh{(\sqrt{\kappa}x)})).
    \end{align}
\end{definition}
\begin{proposition} \label{prop:lorentzian_convergence}
    In the limit where the curvature $\kappa \to 0^+$, $h_{\text{Lorentzian},\kappa}$ recovers $h$:
    \begin{align*}
        \lim_{\kappa \to 0^+} h_{\text{Lorentzian},\kappa}(x) = h(x).
    \end{align*}
\end{proposition}
\begin{proof}
    See Appendix~\ref{prop:proof:lorentzian_convergence}
\end{proof}

So we now have our pre-activations and our activation functions, all that remains is to combine them to compute the output.

\subsection*{Step 4: Compute the Output}
To compute the final output, we use the same approach as \citet{ganeaHyperbolicNeuralNetworks2018, shimizuHyperbolicNeuralNetworks2020}. We use $e^{(i)}$ to denote the canonical basis of $T_{\mathbf{o}} M_{\text{out}}$, and we recall that $H_{e^{(i)}, \mathbf{o}}$ is the hyperplane through the origin that is orthogonal to the $i$-th spatial axis. Then, we construct $\mathbf{y} \in M_{\text{out}}$ such that the distance of $\mathbf{y}$ to $H_{e^{(i)}, \mathbf{o}}$ is $a_i$:

\begin{equation}
    d^{\text{ signed}}_{M}(\mathbf{y}, H_{e^{(i)}, \mathbf{o}}) = a_i,
\end{equation}
for all $i \in \{1, ..., D_{\text{out}}\}$.

\begin{proposition} \label{prop:activations_to_output}
    Given $\mathbf{a} \in \mathbb{R}^{D_\text{out}}$, the vector $\mathbf{y} \in \mathbb{L}_\kappa^{D_\text{out}}$ whose distances to the $D_{\text{out}}$ hyperplanes passing through the origin and orthogonal to the $D_{\text{out}}$ spatial axes are given by $\mathbf{a}$ is:
    \begin{equation}
        \label{eq:output_point}
        \mathbf{y} = 
    \begin{pmatrix}
        y_1 \\
        \overline{\mathbf{y}}
    \end{pmatrix},
    \end{equation}
    such that:
    \begin{align}
        \overline{{\mathbf{y}}} &= \frac{1}{\kappa} \sinh{\left( \sqrt{\kappa} \mathbf{a} \right)}, \quad y_1 = \sqrt{\frac{1}{\kappa} + \|\overline{\mathbf{y}}\|_E^2}.
    \end{align}
\end{proposition}
\begin{proof} 
    See Appendix~\ref{prop:proof:activations_to_output}.
\end{proof}

The Lorentzian linear layer simplifies significantly because the arcsinh used to compute the pre-activations (\Cref{eq:preactivations}) cancels out with the sinh in the manifold mapping (\Cref{eq:output_point}). While the layer is built by composing a distance metric, a Lorentzian activation, and a manifold mapping, these operations collapse into a much more straightforward calculation. 

Ultimately, the spatial output $\overline{\mathbf{y}}$ is obtained by simply applying a scalar non-linearity $h$ to the Lorentz inner products between the input and $\mathbf{v}^{(i)}$, and the exponential nature of hyperbolic geometry comes from the exponential relation between $\mathbf{w}^{(i)}$ and $\mathbf{v}^{(i)}$. 

By organizing the weight vectors $\mathbf{v}^{(i)}$ into an $D_\text{out} \times (D_\text{in}+1)$ matrix $\mathbf{V} = (\mathbf{v}^{(1)}, \dots, \mathbf{v}^{(D_{\text{out}})})^T$, we can define the layer $f_{\text{LorentzFC}}$ as:
\begin{equation}
    f_{\text{LorentzFC}}(\mathbf{x}) = \mathbf{y} = 
    \begin{pmatrix} 
        y_1 \\ 
        \overline{\mathbf{y}}
    \end{pmatrix},
\end{equation}
where
\begin{align}
    \label{eq:compute_output}
    \overline{\mathbf{y}} &= \mathbf{h}(\mathbf{V}\mathbf{I}_{1,D_{\text{in}}}\mathbf{x}), \quad    y_1 = \sqrt{\frac{1}{\kappa} + \|\overline{\mathbf{y}}\|_E^2}.
\end{align}

\subsection{The Solution to the Pathology}
We will now prove that our proposed method has the desired property that it can learn exponentially large distances in linear time.

\begin{proposition}
\label{prop:solve}
    Our linear layer, defined in \Cref{eq:compute_output}, can learn to embed a tree of depth $h$ with low distortion in $\mathcal{O}(h)$ gradient updates.
\end{proposition}
\begin{proof}
    See Appendix~\ref{prop:solve}.
\end{proof}

As such, our layer has the desired property in theory. In \Cref{subsec:toy}, we will verify in a toy experiment that this is also the case in practice.

\subsection{Speed Up Inference: Caching $\mathbf{V}$}
\label{subsec:cache}
Note that our layer can be decomposed in three steps: first, computing the matrix $\mathbf{V}$ from the parameters $\mathbf{W}$ and $\mathbf{b}$; second, compute the spatial components of the output by doing a matrix multiplication and applying the activation function; third, compute the time component to satisfy the hyperboloid constraint.

During training, it is necessary to perform every step as the gradient must propagate to the parameters, however, the computation of $\mathbf{V}$ does not depend on the input $\mathbf{x}$, but solely on the parameters. Since the parameters remain fixed during inference, once the network is trained, we can significantly accelerate inference by precomputing and caching $V$. This avoids the repeated evaluation of hyperbolic sine and cosine functions required to compute $\mathbf{v}^{(i)}$ (as seen in \Cref{eq:defn_p} and \Cref{eq:defn_v}) for every forward pass.

Since the mapping from $(\mathbf{W}, \mathbf{v})$ to $\mathbf{V}$ is invertible (\Cref{prop:cache_inverse}), we can discard the original parameters after training and store only $\mathbf{V}$, providing the full speedup without memory overhead.
\begin{proposition} \label{prop:cache_inverse} 
    Defining $\mathbf{v}^{(i)} = (v^{(i)}_1, \overline{\mathbf{v}}^{(i)})$ as in \cref{eq:defn_v} as a function of $\mathbf{w}^{(i)}$ and $b_i$, the function is invertible for $\|\mathbf{w}^{(i)}\|_E > 0$, and the inverse is given by:
    \begin{align}
        \mathbf{w}^{(i)} &= \frac{\|\mathbf{v}^{(i)}\|_\mathcal{L}}{\|\overline{\mathbf{v}}^{(i)}\|_E} \overline{\mathbf{v}}^{(i)}, \\
        b_i &= - \frac{\|\mathbf{v}^{(i)}\|_\mathcal{L}}{\sqrt{\kappa}} \arcsinh\left( \frac{v^{(i)}_1}{\|\mathbf{v}^{(i)}\|_\mathcal{L}} \right).
    \end{align}
\end{proposition}

\begin{proof}
    See Appendix~\ref{prop:proof:cache_inverse}
\end{proof}

\subsection{Batch Normalization} \label{sec:batchnorm}
Prior works on hyperbolic deep learning use Batch Normalization for its useful properties such as making models easier to train, and regularization \citep{ioffeBatchNormalizationAccelerating2015, chenFullyHyperbolicNeural2022, vanSpengler2023PoincareResNet, bdeir2023FullyVision}. However, the principle behind BatchNorm clashes with the properties of hyperbolic space: exponential growth in embeddings causes an exponential increase in variance, which in turn results in an exponentially large denominator in the scaling of BatchNorm. Conversely, layers with a small output result in exponentially small denominators. In our experience, this results in highly unstable training dynamics. Even so, simply removing BatchNorm makes the network very hard to train.

As an alternative, we look to Weight Normalization + mean-only BatchNorm \citep{salimansWeightNormalizationSimple2016}: we reparametrise the weight vectors as follows: \begin{equation}
    \mathbf{w}^{(i)} = g \frac{\mathbf{a}}{\|\mathbf{a}\|},
\end{equation}
where $g$ is the scale, and $\mathbf{a}$ is an unnormalized weight vector. Combining these two methods makes the network much easier to train, and does not have the same problem with exponentially small or large variances.

\section{Experiments}
\label{sec:experiments}
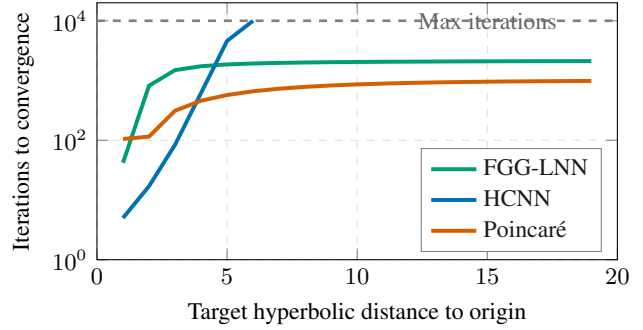
\begin{figure}
    \centering
    \begin{tikzpicture}
\begin{semilogyaxis}[
    width=8.5cm,
    height=5cm,
    xlabel={Target hyperbolic distance to origin},
    ylabel={Iterations to convergence},
    legend pos=south east,
    legend cell align=left,
    legend style={font=\small, fill=white, fill opacity=0.9, draw=black!50},
    grid=major,
    grid style={dashed, gray!25},
    tick label style={font=\small},
    label style={font=\small},
    title style={font=\normalsize},
    xmin=0, xmax=20,
    ymin=1, ymax=20000,
]

\addplot[color=icmlgreen, line width=1.5pt, solid, mark=none] coordinates {
    (1, 42) (2, 818) (3, 1493) (4, 1733) (5, 1859) 
    (6, 1928) (7, 1972) (8, 2003) (9, 2026) (10, 2044) 
    (11, 2058) (12, 2069) (13, 2079) (14, 2087) (15, 2095) 
    (16, 2101) (17, 2106) (18, 2109) (19, 2111)
};
\addlegendentry{FGG-LNN}

\addplot[color=icmlblue, line width=1.5pt, solid, mark=none] coordinates {
    (1, 5) (2, 17) (3, 84) (4, 618) (5, 4603) (6, 10000)
};
\addlegendentry{HCNN}

\addplot[color=icmlorange, line width=1.5pt, solid, mark=none] coordinates {
    (1, 105) (2, 115) (3, 312) (4, 459) (5, 572) 
    (6, 660) (7, 729) (8, 783) (9, 826) (10, 861) 
    (11, 888) (12, 911) (13, 929) (14, 943) (15, 955) 
    (16, 965) (17, 972) (18, 979) (19, 984)
};
\addlegendentry{Poincaré}

\addplot[color=black!50, line width=1pt, dashed, forget plot] coordinates {
    (0, 10000) (20, 10000)
};
\node[anchor=west, font=\footnotesize, text=black!60] at (axis cs:12,10000) {Max iterations};

\end{semilogyaxis}
\end{tikzpicture}
    \caption{Number of iterations required to fit a hyperbolic hyperplane with a given hyperbolic distance to the origin.}
    \label{fig:toy}
\end{figure}

\subsection{Geometric Analysis}
\label{subsec:toy}
We first set up a toy experiment to study the geometric effect of our layer; we measure the number of optimization steps needed to fit hyperbolic hyperplanes with increasing hyperbolic distances to the origin. The idea behind this is the following: embedding a hierarchy with low distortion requires consistent spacing between embedded parent-child pairs \citep{sarkar2012}, which means linear growth in total distance of the root node to the leaves. It follows that if a linear layer cannot efficiently fit hyperbolic hyperplanes whose distance from the origin grows linearly, it also cannot embed hierarchies with low distortion. 

As predicted in \Cref{sec:pathology}, we see that, for the default Lorentz linear layer proposed by \citet{chenFullyHyperbolicNeural2022}, the necessary number of optimization steps grows exponentially with hyperbolic distance. For the Poincaré linear layer and FGG-LNN, the number of required steps starts off higher, but from there grows linearly rather than exponentially. The Poincaré linear layer requires fewer iterations than our linear layer, because the Poincaré model has the advantage that linear growth in hyperbolic norm does not correspond to exponential growth in Euclidean norm, rather, the growth in Euclidean norm is sublinear. Still, we can conclude from this toy experiment that our FGG-LNN solves the pathology suffered by other Lorentz linear layers.

But, this begs the question: how can networks like HCNN \citep{bdeir2023FullyVision} still achieve state-of-the-art accuracy? The HCNN MLR layer \textit{is} formulated using distances to hyperplanes, which means that logits grow linearly with hyperbolic distance. On the other hand, as we proved in \Cref{sec:pathology}, their linear layers can only grow hyperbolic distances logarithmically. To answer this question, we investigate what happens to the hyperbolic distance to the origin of the layer outputs. What we see is that the linear layers in HCNN do not significantly grow the space (\Cref{fig:spatial_norm}), as predicted. Rather, the hyperbolic distance remains more or less constant, until the last layer. Here, the linear layer learns to give an output that is very close to the origin, and the BatchNorm scale parameter (gamma) learns to grow, culminating in a blowup of the hyperbolic distance from ~0.5 to ~9. By contrast, in our network, the hyperbolic distance steadily grows over the network layers, with the BatchNorm consistently pulling it down somewhat through its centering, and the linear layers then growing the output norm again.

\begin{figure}
    \centering
    \begin{tikzpicture}
    \begin{axis}[
        width=8.5cm,
        height=6cm,
        font=\rmfamily,
        title style={font=\bfseries},
        xlabel style={font=\small},
        ylabel style={font=\small},
        yticklabel style={text width=1.25em, align=left},
        tick label style={font=\footnotesize},
        legend style={
            font=\scriptsize,
            at={(0.02,0.98)},
            anchor=north west,
            draw=black!50,
            fill=white,
            fill opacity=0.9
        },
        xlabel={Layer index},
        ylabel={Mean hyperbolic distance to origin},
        ymax=13,
        grid=major,
        grid style={dashed, gray!30},
        axis line style={draw=black!80},
        cycle list name=color list,
        xtick={0,5,10,15,19},
    ]

    \addplot[
        color=icmlblue,
        mark=*,
        mark size=1.5pt,
        line width=1.0pt,
        solid
    ] coordinates {
        (0, 2.4213) (1, 2.3926) (2, 2.0756) (3, 2.5855) (4, 1.2486) 
        (5, 1.7739) (6, 2.8280) (7, 2.4195) (8, 2.6041) (9, 1.7782) 
        (10, 1.5828) (11, 2.6339) (12, 2.4201) (13, 2.6339) (14, 1.8044) 
        (15, 1.5312) (16, 2.9002) (17, 2.2975) (18, 2.6346) (19, 0.5360)
    };
    \addlegendentry{Theirs - Linear}

    \addplot[
        color=icmlblue,
        mark=square*,
        mark size=1.5pt,
        line width=1.0pt,
        dashed,
        mark options={solid}
    ] coordinates {
        (0, 1.4742) (1, 1.1457) (2, 1.3559) (3, 1.3624) (4, 1.2665) 
        (5, 1.2421) (6, 1.6775) (7, 1.5662) (8, 1.9490) (9, 1.4067) 
        (10, 1.5326) (11, 2.0529) (12, 2.2039) (13, 2.7249) (14, 2.0940) 
        (15, 1.0449) (16, 2.0100) (17, 2.8535) (18, 2.1366) (19, 9.0000)
    };
    \addlegendentry{Theirs - BatchNorm}

    \addplot[
        color=icmlgreen,
        mark=*,
        mark size=1.5pt,
        line width=1.0pt,
        solid
    ] coordinates {
        (0, 2.4760) (1, 3.6905) (2, 3.2995) (3, 4.5494) (4, 5.1429) 
        (5, 6.3260) (6, 6.9738) (7, 5.5432) (8, 7.4261) (9, 7.0450) 
        (10, 8.1946) (11, 8.6862) (12, 7.1797) (13, 9.2319) (14, 8.7198) 
        (15, 9.4076) (16, 8.9741) (17, 8.3374) (18, 9.1496) (19, 9.0856)
    };
    \addlegendentry{Ours - Linear}

    \addplot[
        color=icmlgreen,
        mark=square*,
        mark size=1.5pt,
        line width=1.0pt,
        dashed,
        mark options={solid}
    ] coordinates {
        (0, 2.5856) (1, 2.7492) (2, 2.7391) (3, 3.9767) (4, 4.8586) 
        (5, 5.8697) (6, 6.3916) (7, 4.9079) (8, 6.8820) (9, 6.4763) 
        (10, 7.7264) (11, 8.2245) (12, 6.9537) (13, 8.9232) (14, 8.3803) 
        (15, 9.2266) (16, 8.7536) (17, 8.2495) (18, 8.9776) (19, 8.8941)
    };
    \addlegendentry{Ours - BatchNorm}

    \end{axis}
\end{tikzpicture}
    \caption{Growth of the hyperbolic distance to the origin over the network's layers. Our FGG-LNN grows the distance gradually as depth increases. HCNN has the distance remain relatively constant, until the final BatchNorm layer, where the distance explodes to ~9.}
    \label{fig:spatial_norm}
\end{figure}
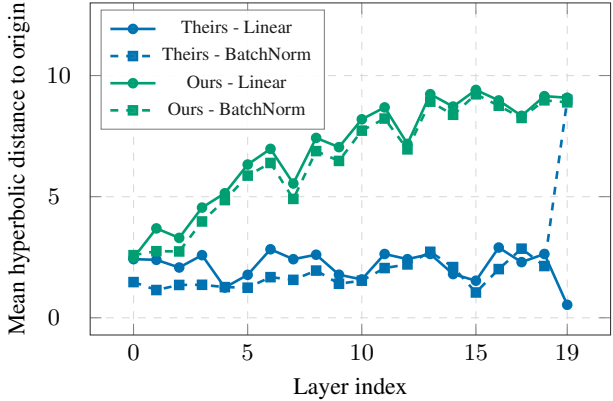

\subsection{Image Classification}
Apart from the geometric interpretation, we validate that our network is still able to learn well, given its completely new architecture. For this, we perform some minor experiments on CIFAR-10 and CIFAR-100 using the same hyperparameters as HCNN \citet{bdeir2023FullyVision}, and compare the accuracies to a standard Euclidean ResNet18 \citep{heDeepResidualLearning2016} and to HCNN (\Cref{tab:cifar}). 

\begin{table}[]
\centering
\caption{Test accuracy, averaged across three runs. Accuracies for all three models are comparable, indicating that the efficiency gains of FGG-LNN do not come at the cost of accuracy.}
\label{tab:cifar}
\begin{tabular}{lcc}
\toprule
\textbf{Method}    & \multicolumn{1}{c}{\textbf{CIFAR-10}} & \textbf{CIFAR-100}                     \\
\midrule
FGG-LNN   & 94.66 $\pm$ 0.04 & 76.14 $\pm$ 0.11 \\
HCNN      & 94.61 $\pm$ 0.25 & 76.09 $\pm$ 0.23 \\
Euclidean & 94.58 $\pm$ 0.41 & 76.65 $\pm$ 0.24 \\
\bottomrule
\end{tabular}
\end{table}

We see that a neural network using our linear layer achieves similar performance to a standard Euclidean ResNet18 and HCNN. We note that the numbers we report for HCNN are slightly lower than in their paper; we use a validation set for checkpoint selection, whereas the original work selected checkpoints based on test accuracy.

\textbf{Performance}

To compare the performance of our proposed method to prior work, we visualise the training time and inference speed of four different networks: a Euclidean baseline ResNet-18 \citep{heDeepResidualLearning2016}, a Poincaré ResNet-20 \citep{vanSpengler2023PoincareResNet}, a Lorentz ResNet-18 as implemented by \citet{bdeir2023FullyVision}, and our FGG-LNN. For FGG-LNN, during inference, we cache $\mathbf{V}$ as described in \Cref{subsec:cache}.

We compare the performance of our model in \Cref{tab:runtime}. Our model substantially closes the gap between hyperbolic models and Euclidean models: our training times are 2.9x as fast as HCNN \citep{bdeir2023FullyVision}, and 8.3x as fast as Poincaré ResNet \citep{vanSpengler2023PoincareResNet}. This allows us to train a FGG-LNN ResNet18 model on CIFAR-100 in around 70 minutes, and to our knowledge, it is the fastest-training fully hyperbolic model that exists. For inference times, we employ the caching strategy for our model described in \Cref{subsec:cache}. For moderate input and output dimensions (16 to 256), our inference times are 3x as fast as the linear layer formulated by \citeauthor{chenFullyHyperbolicNeural2022}, and 6.8x as fast as the Poincaré linear layer \citep{shimizuHyperbolicNeuralNetworks2020, vanSpengler2023PoincareResNet}. For an input and output dimension of 4096, the effect decreases, as more of the total computation time is being taken up by the matrix multiplication inside the linear layer.

Altogether, our network achieves similar performance as previous Lorentz neural networks in approximately 1/3 of the training time of previous Lorentz neural networks, and also runs inference in approximately 1/3 of the time.

\begin{table}[]
\caption{Runtime comparison as multiples of the Euclidean baseline. Inference times measured for a single linear layer; training times for ResNet18 on CIFAR-100.}
\label{tab:runtime}
\small
\begin{tabular}{l ccc c}
\toprule
 & \multicolumn{3}{c}{Inference} & Training \\
\cmidrule(lr){2-4} \cmidrule(lr){5-5}
Method & $d=16$ & $d=256$ & $d=4096$ & ResNet18 \\
\midrule
FGG-LNN  & \textbf{2.39x} & \textbf{2.10x} & \textbf{1.06x} & \textbf{4.24x} \\
HCNN     & 7.06x          & 6.18x          & 1.26x          & 12.25x \\
Poincaré & 16.34x         & 14.37x         & 1.94x          & 35.21x \\
\bottomrule
\end{tabular}
\end{table}

\section{Conclusion}
In this paper, we identified a pathology in existing Lorentz fully connected layers where the hyperbolic norm of the output scales logarithmically with the number of gradient descent steps. To solve this problem, we propose a fully connected layer in the Lorentz model based on the distance to hyperplanes interpretation. We show theoretically and experimentally that our design solves this problem. We investigate how this design interacts with Lorentz BatchNorm, and propose to use WeightNorm + centering-only BatchNorm to stabilise training. Furthermore, by introducing the Lorentzian activation functions and an efficient caching strategy, we achieve the fastest inference times amongst hyperbolic models. This efficacy, along with several other optimisations, translates to 70-minute training times for a Lorentz ResNet18 architecture in CIFAR-100. We make our code available in the hope that it will accelerate future hyperbolic deep learning research. We anticipate that future research into native Lorentz normalization and adaptive curvature will further enhance the flexibility of these networks, ultimately unlocking the full representational power of hyperbolic spaces for increasingly complex, hierarchical datasets.



\section*{Impact Statement}
This paper presents work whose goal is to advance the field
of Machine Learning. There are many potential societal
consequences of our work, none which we feel must be
specifically highlighted here.

\bibliography{main}
\bibliographystyle{icml2026}

\newpage
\appendix
\onecolumn
\section{Hyperbolic Geometry in the Lorentz Model}
\label{appendix:hyperbolic_ops}
Vectors in the Minkowski ambient space $\mathbb{R}^{n+1}$ of the Lorentz model $\mathbb{L}_{\kappa}^{n}$ can be classified in three categories depending in their square norm $\|v\|_\mathcal{L}$: \textbf{Timelike:} if $\|v\|_\mathcal{L} < 0$; \textbf{Lightlike:} if $\|v\|_\mathcal{L} = 0$; \textbf{Spacelike:} if $\|v\|_\mathcal{L} > 0$. 
Notice that vectors in $\mathbb{L}_{\kappa}^{n}$ are timelike, and vectors in the tangent space are spacelike. 

To perform deep learning in hyperbolic space, we must define operations that allow us to transition between the manifold $\mathbb{L}_\kappa^n$ and its tangent spaces $T_{\mathbf{x}} \mathbb{L}_\kappa^n$, as well as operations to move tangent vectors between different points on the manifold. In the Lorentz model, these operations have closed-form expressions.

The \textbf{exponential map} $\text{Exp}_{\mathbf{x}}(\mathbf{v})$ projects a tangent vector $\mathbf{v} \in T_{\mathbf{x}} \mathbb{L}_\kappa^n$ onto the manifold. Intuitively, it moves a distance $\|\mathbf{v}\|_{\mathcal{L}}$ along the geodesic starting at $\mathbf{x}$ in the direction of $\mathbf{v}$. For $\mathbf{v} \neq \mathbf{0}$, it is defined as:
\begin{equation} \label{eq:expmap}
    \text{Exp}_{\mathbf{x}}(\mathbf{v}) = \cosh(\sqrt{\kappa} \|\mathbf{v}\|_{\mathcal{L}}) \mathbf{x} + \frac{\sinh(\sqrt{\kappa} \|\mathbf{v}\|_{\mathcal{L}})}{\sqrt{\kappa} \|\mathbf{v}\|_{\mathcal{L}}} \mathbf{v}.
\end{equation}
When $\mathbf{v} = \mathbf{0}$, $\text{Exp}_{\mathbf{x}}(\mathbf{0}) = \mathbf{x}$.

The \textbf{logarithmic map} $\text{Log}_{\mathbf{x}}(\mathbf{y})$ is the inverse of the exponential map. It maps a point $\mathbf{y} \in \mathbb{L}_\kappa^n$ back to the tangent space $T_{\mathbf{x}} \mathbb{L}_\kappa^n$. The resulting vector $\mathbf{v}$ represents the direction and distance from $\mathbf{x}$ to $\mathbf{y}$ in the tangent space:
\begin{equation} \label{eq:logmap}
    \text{Log}_{\mathbf{x}}(\mathbf{y}) = \frac{d(\mathbf{x}, \mathbf{y})}{\sinh(\sqrt{\kappa} d(\mathbf{x}, \mathbf{y}))} (\mathbf{y} + \kappa (\mathbf{x} \circ \mathbf{y}) \mathbf{x}),
\end{equation}
where $d(\mathbf{x}, \mathbf{y})$ is the geodesic distance defined in the preliminaries.

The \textbf{parallel transport} $\Gamma_{\mathbf{x} \to \mathbf{y}}(\mathbf{v})$ maps a vector $\mathbf{v} \in T_{\mathbf{x}} \mathbb{L}_\kappa^n$ to $T_{\mathbf{y}} \mathbb{L}_\kappa^n$ along the unique geodesic connecting $\mathbf{x}$ and $\mathbf{y}$ while preserving the Lorentz inner product, and thus the norm.

In the Lorentz model, the parallel transport of a vector $\mathbf{v}$ from $\mathbf{x}$ to $\mathbf{y}$ is given by:
\begin{equation} \label{eq:parallel_transport}
    \Gamma_{\mathbf{x} \to \mathbf{y}}(\mathbf{v}) = \mathbf{v} + \frac{\kappa(\mathbf{y} \circ \mathbf{v})}{1 - \kappa(\mathbf{x} \circ \mathbf{y})} (\mathbf{x} + \mathbf{y}).
\end{equation}
This operation is crucial for operations like weight updates or moving gradient information across the manifold.

\section{Geometric Interpretation of the Hyperplane}
\label{appendix:geometric_interpretation}

To complement the algebraic derivation of the Lorentz linear layer, we provide a visual intuition of the decision boundaries in Figure~\ref{fig:3D_2D_LNN}. While the algebraic formulation relies on ambient vectors and inner products, the geometric construction relies on the intersection of the manifold with specific subspaces.

\begin{figure*}[htbp]
    \centering
    \begin{subfigure}[b]{0.3\textwidth}
        \centering
        \includegraphics[width=\textwidth]{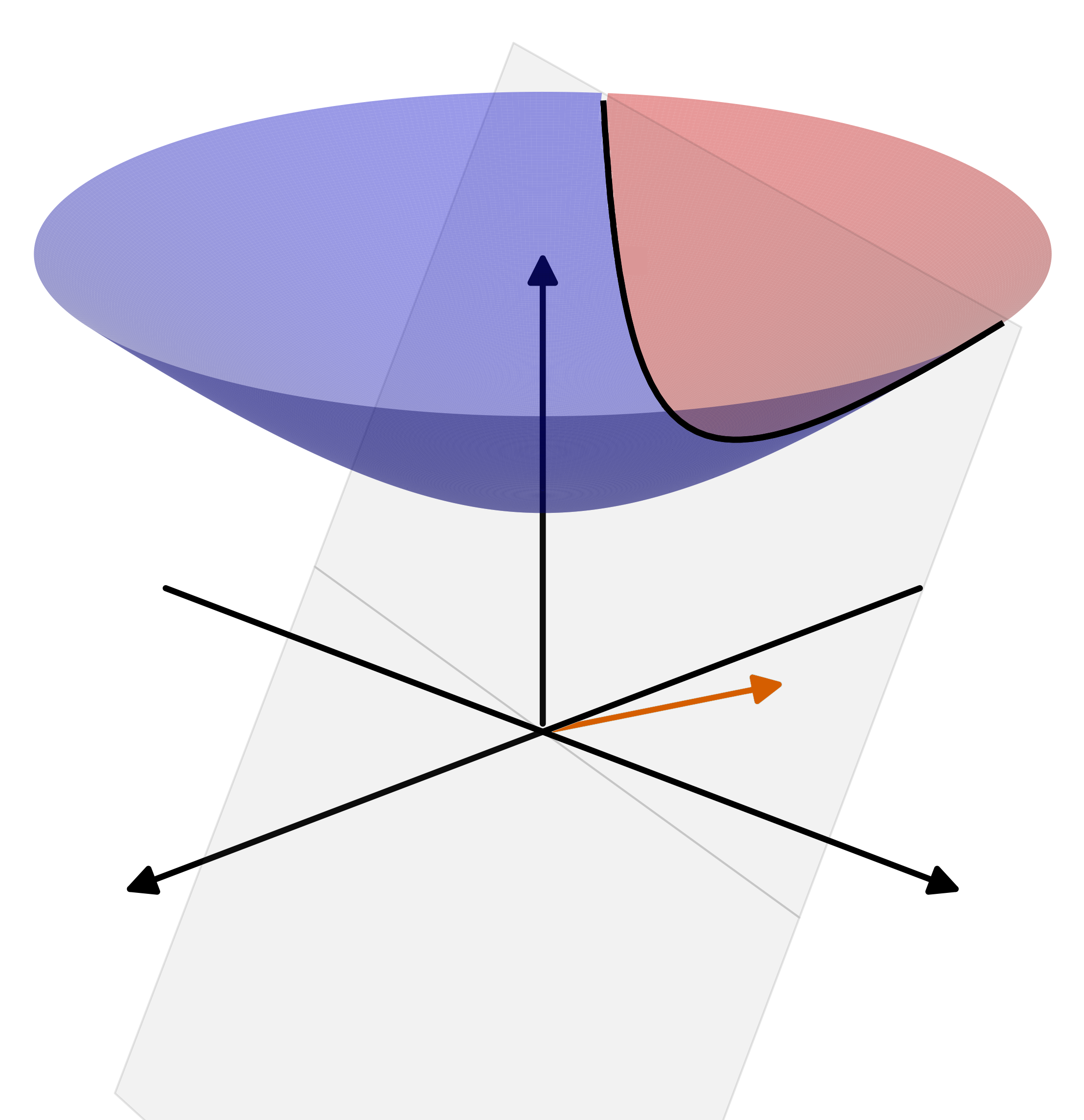}
        \caption{}
        \label{fig:img1}
    \end{subfigure}
    \hfill 
    \begin{subfigure}[b]{0.3\textwidth}
        \centering
        \includegraphics[width=\textwidth]{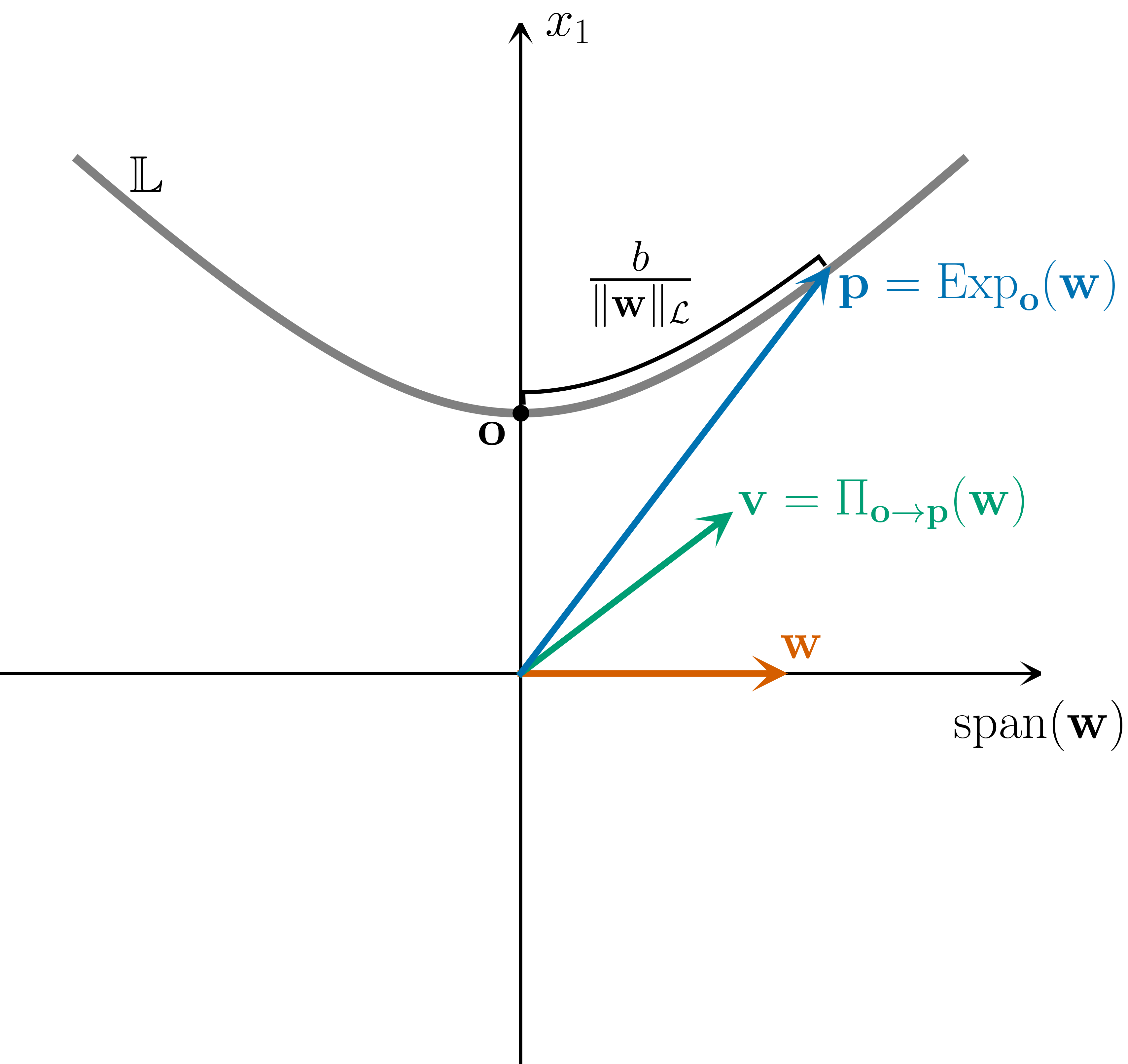}
        \caption{}
        \label{fig:img2}
    \end{subfigure}
    \hfill
    \begin{subfigure}[b]{0.3\textwidth}
        \centering
        \includegraphics[width=\textwidth]{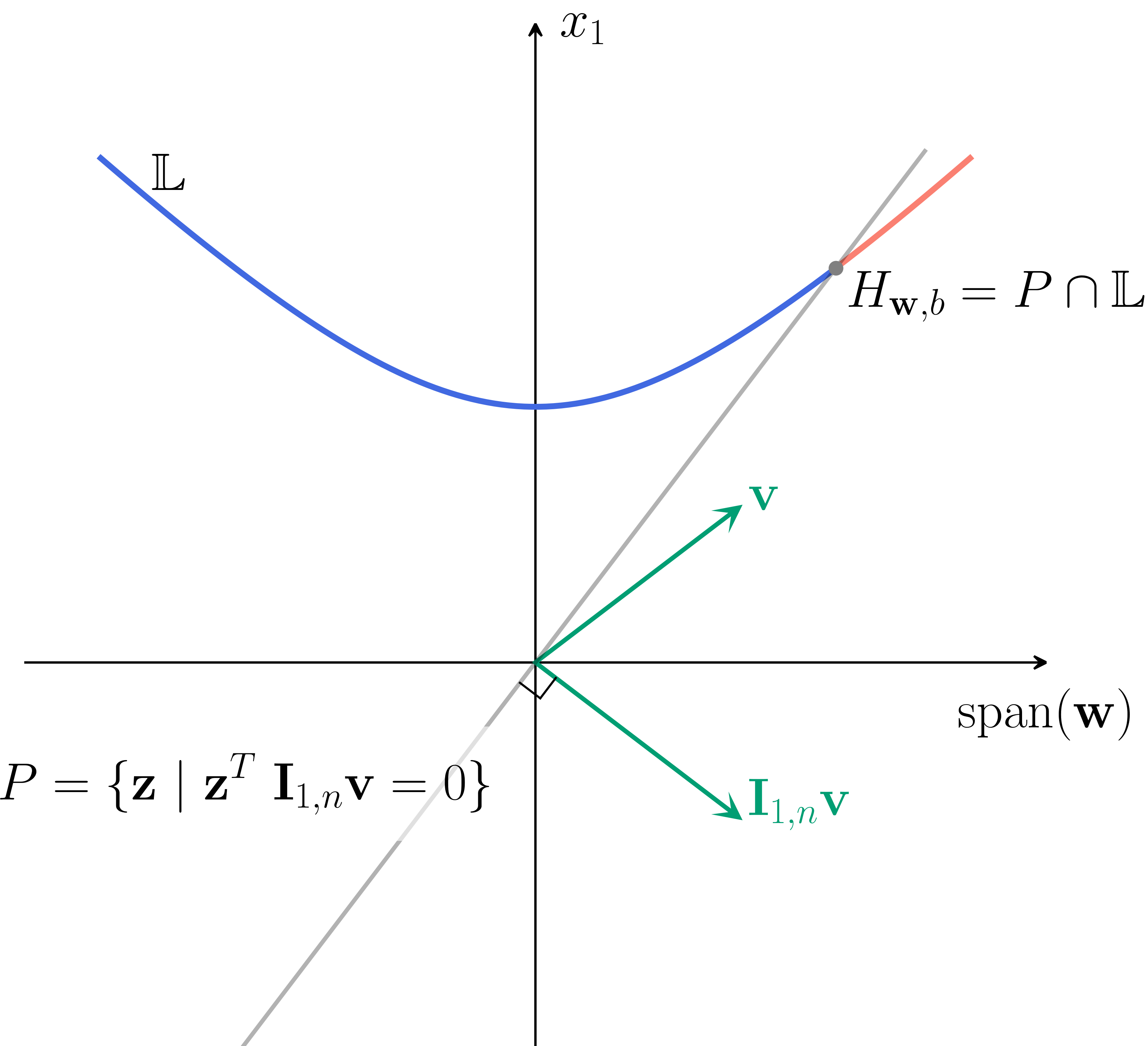}
        \caption{}
        \label{fig:img3}
    \end{subfigure}
    
    \caption{\textbf{Geometric construction of the Lorentz hyperplanes.} \textbf{(a)} A 3D visualization of the Lorentz model $\mathbb{L}^2$ and geodesic hyperplane with discriminative regions in blue and red. The orange vector is the weight parameter $\mathbf{w}$ of the hyperplane. The gray plane $P$ is the set of points in the ambient space Minkowski orthogonal to $\mathbf{v}$.
    \textbf{(b)} and \textbf{(c)} are planar sections of the hyperboloid in \textbf{(a)} intersecting it with the subspace $\text{span}(\mathbf{e}_1, \mathbf{w})$. This yields a one-dimensional hyperbola.
    \textbf{(b)} Parametrization of the hyperplane reference point: the weight vector $\mathbf{w}$ and bias $b$ define an anchor point $\mathbf{p}$ on the manifold via the exponential map. The distance from the origin $\mathbf{o}$ to $\mathbf{p}$ is $\frac{b}{\|\mathbf{w}\|_\mathcal{L}}$.
    \textbf{(c)} The simplified hyperplane condition: the geodesic hyperplane $H_{\mathbf{w},b}$ is the intersection of the hyperboloid with the plane $P$ that are orthogonal to the parallel-transported weight vector v in the ambient Minkowski space $P=\{\mathbf{z} \ | \ \mathbf{z} \circ \mathbf{v} = \mathbf{z} ^T \mathbf{I}_{1,n} \mathbf{v}=0\}$.}
    \label{fig:3D_2D_LNN}
\end{figure*}

\paragraph{The Reference Point (Panel b).} 
As defined in \cref{eq:defn_p}, the bias $b$ and the weight vector $\mathbf{w}$ determine a specific reference point $\mathbf{p}$ on the manifold. Geometrically, this corresponds to starting at the origin $\mathbf{o}$ and shooting a geodesic in the direction of the weights $\mathbf{w}$. The magnitude of the displacement is determined by the ratio $b / \|\mathbf{w}\|_{\mathcal{L}}$. This mechanism effectively ``pushes" the decision boundary away from the origin, analogous to how the bias term in Euclidean geometry translates a hyperplane from the origin.

\paragraph{The Orthogonality Condition (Panel c).}
Once the reference point $\mathbf{p}$ is established, the orientation of the decision boundary is defined by the normal vector. However, because the tangent space $T_{\mathbf{p}}\mathbb{L}_\kappa^n$ changes at every point, we cannot simply use $\mathbf{w}$ as the normal vector at $\mathbf{p}$. Instead, we must parallel transport $\mathbf{w}$ from the origin to $\mathbf{p}$, resulting in the vector $\mathbf{v}$.
    
Panel (c) illustrates the simplification shown in Proposition~\ref{prop:lorentz_hyperplane_simplification}. The geodesic hyperplane (the blue curve) is exactly the intersection of the hyperbolic manifold with a Euclidean subspace $P$ (the gray line) passing through the origin. This subspace $P$ is defined as the set of all ambient points Minkowski-orthogonal to $\mathbf{v}$. This highlights a key property of the Lorentz model: geodesic submanifolds can be represented as linear intersections in the ambient space.

\section{Proofs} \label{appendix:proofs}
Here we provide the formal derivations and proofs for the propositions and theorems discussed in the main text. The following section establishes the geometric foundations of our Lorentzian framework, characterizes the representation capacity of hyperbolic layers, and derives the practical implementation identities required for efficient computation.

\begin{restatedproposition}
{prop:upperbound}

    Let $\mathbf{W}_n$ be a weight matrix after $n$ steps of optimization with updates bounded in Frobenius norm by $\delta$ (i.e., $\|\Delta \mathbf{W}_t\|_F \leq \delta$). The amount by which $\mathbf{W}_n$ can stretch the space is upper bounded by $O(\ln n)$: we have
    $$d_\mathcal{H}(\mathbf{o}, \mathbf{y}) - d_\mathcal{H}(\mathbf{o}, \mathbf{x}) = \mathcal{O}(\ln n).$$
\end{restatedproposition}
\begin{proof}
    Let $\mathbf{x} = (x_1, \bar{\mathbf{x}}),$ with $x_1 = \sqrt{\|\bar{\mathbf{x}}\|^2 + \frac{1}{\kappa}}$. The Chen layer maps this to $\mathbf{y} = (y_1, \bar{\mathbf{y}})$ where $\bar{\mathbf{y}} = \mathbf{W}\mathbf{x}$ and $y_1 = \sqrt{\|\mathbf{W} \mathbf{x}\|^2 + \frac{1}{\kappa}}$. Using $\text{arccosh}(z) = \ln(z + \sqrt{z^2-1}),$ we can write:
    $$d_\mathcal{H}(\mathbf{o}, \mathbf{y}) - d_\mathcal{H}(\mathbf{o}, \mathbf{x}) = \frac{1}{\sqrt{\kappa}} \ln \left(\frac{\sqrt{\kappa} y_1 + \sqrt{\kappa y_1^2 - 1}}{\sqrt{\kappa} x_1 + \sqrt{\kappa x_1^2 - 1}}\right).$$
    For any $z \geq 1$, we have $z \leq z + \sqrt{z^2 - 1} \leq 2z$. Applying the upper bound on the numerator, and the lower bound on the denominator, we get:
    $$d_\mathcal{H}(\mathbf{o}, \mathbf{y}) - d_\mathcal{H}(\mathbf{o}, \mathbf{x}) \leq \frac{1}{\sqrt{\kappa}} \ln \left(\frac{2 \sqrt{\kappa}y_1}{\sqrt{\kappa} x_1}\right) = \frac{1}{\sqrt{\kappa}}\left(\ln 2 + \ln \frac{y_1}{x_1} \right).$$
    Now we bound the ratio $\frac{y_1}{x_1}$. From the hyperboloid constraint $x_1^2 - \|\bar{\mathbf{x}}\|^2 = \frac{1}{\kappa}$: 
    $$\|\mathbf{x}\|^2 = x_1^2 + \|\bar{\mathbf{x}}\|^2 = 2x_1^2 - \frac{1}{\kappa}.$$
    Using $y_1^2 = \|\mathbf{W}\mathbf{x}\|^2 + \frac{1}{\kappa} \leq \|\mathbf{W}\|_F^2 \|\mathbf{x}\|^2 + \frac{1}{\kappa}$:
    $$y_1^2 \leq \|\mathbf{W}\|_F^2 (2x_1^2 - \frac{1}{\kappa}) + \frac{1}{\kappa} = 2\|\mathbf{W}\|^2_F x_1^2 - \frac{\|\mathbf{W}\|_F^2 - 1}{\kappa}.$$

    For $\|\mathbf{W}\|_F \geq 1$, the last term is non-positive, so $y_1^2 \leq 2\|\mathbf{W}\|_F^2 x_1^2$, giving:
    $$\frac{y_1}{x_1} \leq \sqrt{2}\|\mathbf{W}\|_F.$$
    Using the triangle inequality $$\|\mathbf{W}_n\|_F \leq \|\mathbf{W}_0\|_F + n\delta:$$
    $$d_\mathcal{H}(\mathbf{o}, \mathbf{y}) - d_\mathcal{H}(\mathbf{o}, \mathbf{x}) \leq \frac{1}{\sqrt{\kappa}}\left(\ln 2 + \ln \sqrt{2} + \ln(\|\mathbf{W}_0\|_F + n\delta)\right) = \mathcal{O}(\ln n).$$
\end{proof}

\begin{restatedproposition}
{prop:treedistance}
    Embedding an $m$-ary tree with low distortion in $\mathbb{L}_\kappa^D$ requires a minimum distance between nodes and their children of
    $$s = \Omega\left(\frac{\ln m}{(D-1)\sqrt{\kappa}}\right).$$
\end{restatedproposition}
\begin{proof}
    In a tree, all parent-child distances equal $1$. For a low-distortion embedding, these must map to approximately equal hyperbolic distances; call this distance $s$. Consequently, all nodes at depth $h$ lie on a hypersphere of radius $R_h = sh$ centered at the root.

    At depth $h$, there are $m^h$ nodes. For low distortion, sibling nodes (which are at tree distance $2$ from each other) must be separated by a hyperbolic distance of at least $\Omega(s)$. This means each node occupies a geodesic ball of radius $\Omega(s)$ on the hypersphere, and these balls must be disjoint.

    In a $D$-dimensional hyperbolic space $\mathbb{L}^D_\kappa$, the surface area of a hypersphere with radius $R$ is \citep{ratcliffe2019FoundationsManifolds}:
    \begin{align*}
        A(R) &= \frac{2\pi^{D/2}}{\Gamma(D/2)} \frac{\sinh^{D-1}(\sqrt{\kappa}R)}{{\kappa^{(D-1)/2}}}.
    \end{align*}

    For large $R$, using $\sinh(x) \approx \frac{1}{2}e^x:$

$$A(R) = \Theta(\exp((D-1)\sqrt{\kappa}R)),$$ where we treat $D$ and $\kappa$ as fixed.

Each node must be separated from its siblings by hyperbolic distance at least $\Omega(s).$ Therefore, around each node we can place a disjoint geodesic ball of radius $\Omega(s)$ on the hypersphere. The area of each ball is some constant $c > 0$ depending on $s, D$ and $\kappa$. Since there are $m^h$ nodes at depth $h$, we require:
$$A(R_h) \geq c \cdot m^h = \Omega(m^h).$$

Substituting:
$$\exp((D-1)\sqrt{\kappa} \cdot sh) = \Omega(m^h).$$
Taking logarithms:
$$(D-1)\sqrt{\kappa} \cdot sh = \Omega(h \ln m).$$
Dividing by $h$:
$$s = \Omega\left(\frac{\ln m}{(D-1)\sqrt{\kappa}}\right).$$
\end{proof}

\begin{restatedproposition}
    {cor:chen} 
    The linear layer formulation by \citet{chenFullyHyperbolicNeural2022} requires $\Omega(e^h)$ gradient updates to embed a hierarchy of depth $h$ with low distortion.
\end{restatedproposition}

\begin{proof}
    For embedding a tree, we place the root at the origin. From \Cref{prop:upperbound}, after $n$ gradient steps, the layer can increase hyperbolic distance from the origin by at most $\mathcal{O}(\ln n)$. Since the root is at distance $0$, the maximum reachable distance is $\mathcal{O}(\ln n)$. 
    
    From \Cref{prop:treedistance}, embedding an $m$-ary tree of depth $h$ with low distortion requires nodes at depth $h$ to lie at hyperbolic distance at least
    $$R_h = sh = \Omega\left(\frac{h \ln m}{(D-1)\sqrt{\kappa}}\right) = \Omega(h).$$
    Thus, for the embedding to succeed, we need the reachable distance to be at least $R_h$: 
    $$\mathcal{O}(\ln n) \geq \Omega(h),$$
    which requires $\ln n = \Omega(h),$
    so
    $$n = \Omega(e^h).$$
\end{proof}

\begin{restatedproposition}{prop:lorentz_hyperplane_simplification}
    In the Lorentz model, the condition for the geodesic hyperplane defined in \cref{eq:general_hyperplane}, given by $\left\langle \text{Log}_{\mathbf{p}_i}(\mathbf{z}), \Gamma_{\mathbf{o} \to \mathbf{p}_i}(\mathbf{w}^{(i)}) \right\rangle_{\mathbf{p}_i} = 0$, simplifies to the ambient orthogonality condition:
    \begin{equation}
        \label{eq:lorentz_hyperplane}
        H_{\mathbf{w}^{(i)}, b_i} = \left\{ \mathbf{z} \in \mathbb{L}_{\kappa}^{n} \mid \mathbf{z} \circ \mathbf{v}^{(i)} = 0 \right\}
    \end{equation}
    where $\mathbf{v}^{(i)} = \Gamma_{\mathbf{o} \to \mathbf{p}_i}(\mathbf{w}^{(i)})$.
\end{restatedproposition}

\begin{proof}
    \label{prop:proof:lorentz_hyperplane_simplification}
    The hyperplane is defined as the set of points $\mathbf{z}$ such that the tangent vector pointing from $\mathbf{p}_i$ to $\mathbf{z}$ is orthogonal to the normal vector $\mathbf{v}^{(i)}$.
    In the Lorentz model, the logarithmic map $\text{Log}_{\mathbf{p}_i}(\mathbf{z})$ is proportional to the projection of $\mathbf{z}$ onto the tangent space $T_{\mathbf{p}_i}\mathbb{L}_\kappa^n$:
    \begin{equation*}
        \text{Log}_{\mathbf{p}_i}(\mathbf{z}) = \alpha \left( \mathbf{z} + \kappa (\mathbf{p}_i \circ \mathbf{z}) \mathbf{p}_i \right)
    \end{equation*}
    where $\alpha \neq 0$ is a scalar coefficient depending on the distance between $\mathbf{p}_i$ and $\mathbf{z}$. Using the bilinearity of the Lorentz inner product, the orthogonality condition becomes:
    \begin{align}
        \text{Log}_{\mathbf{p}_i}(\mathbf{z}) \circ \mathbf{v}^{(i)} &= 0 \\
        \implies \left( \mathbf{z} + \kappa (\mathbf{p}_i \circ \mathbf{z}) \mathbf{p}_i \right) \circ \mathbf{v}^{(i)} &= 0 \\
        \mathbf{z} \circ \mathbf{v}^{(i)} + \kappa (\mathbf{p}_i \circ \mathbf{z}) (\mathbf{p}_i \circ \mathbf{v}^{(i)}) &= 0 \label{eq:proof_distributive}
    \end{align}
    
    By definition, $\mathbf{v}^{(i)}$ is a vector in the tangent space at $\mathbf{p}_i$ (i.e., $\mathbf{v}^{(i)} \in T_{\mathbf{p}_i} \mathbb{L}_\kappa^n$). A property of the Lorentz model is that tangent vectors are orthogonal in the ambient space to their base point 
    implying $\mathbf{p}_i \circ \mathbf{v}^{(i)} = 0$.
    Substituting this into \cref{eq:proof_distributive}, the second term vanishes, yielding:
    \begin{equation}
        \mathbf{z} \circ \mathbf{v}^{(i)} = 0
    \end{equation}
\end{proof}

\begin{restatedtheorem}{theorem:point_to_hyperplane_distance}
    In the Lorentz model, the hyperbolic distance from a point $\mathbf{x} \in \mathbb{L}_\kappa^n$ to the hyperplane $H_{\mathbf{w}^{(i)}, b_i} = \{ \mathbf{z} \in \mathbb{L}_\kappa^n \mid \mathbf{z} \circ \mathbf{v}^{(i)} = 0 \}$ is given by:
    \begin{equation}
        d(\mathbf{x}, H_{\mathbf{w}^{(i)}, b_i}) = \frac{1}{\sqrt{\kappa}} \text{arcsinh} \left( \sqrt{\kappa} \frac{|\mathbf{x} \circ \mathbf{v}^{(i)}|}{\|\mathbf{v}^{(i)}\|_\mathcal{L}}\right)
    \end{equation}
\end{restatedtheorem}

\begin{proof} 
    \label{theorem:proof:point_to_hyperplane_distance}
    Let $\mathbf{y} \in H_{\mathbf{w}^{(i)}, b_i}$ be the point on the hyperplane that minimizes the distance to $\mathbf{x}$. In a Riemannian manifold, the shortest path between a point and a closed submanifold is realized by a geodesic that is orthogonal to the submanifold at the point of intersection \citep{chavelRiemannianGeometryAModernIntroduction2006}. Consequently, the tangent vector of the geodesic connecting $\mathbf{x}$ to $\mathbf{y}$ at the point $\mathbf{y}$ must be aligned with the normal vector $\mathbf{v}^{(i)}$. 
    
    This implies that the geodesic is contained in a hyperplane in the ambient Minkowski space spanned by $\mathbf{x}$ and $\mathbf{v}^{(i)}$ \citep{ratcliffe2019FoundationsManifolds}. We can therefore express the unnormalized vector $\tilde{\mathbf{y}}$ pointing to $\mathbf{y}$ as a linear combination:
    \begin{equation}
        \tilde{\mathbf{y}} = \mathbf{x} + \alpha \mathbf{v}^{(i)}
    \end{equation}
    for some scalar $\alpha \in \mathbb{R}$. Since $\mathbf{y}$ lies on the hyperplane, it must satisfy the orthogonality condition $\tilde{\mathbf{y}} \circ \mathbf{v}^{(i)} = 0$. Substituting the linear combination:
    \begin{align}
        (\mathbf{x} + \alpha \mathbf{v}^{(i)}) \circ \mathbf{v}^{(i)} &= 0 \\
        \mathbf{x} \circ \mathbf{v}^{(i)} + \alpha (\mathbf{v}^{(i)} \circ \mathbf{v}^{(i)}) &= 0 \\
        \implies \alpha &= - \frac{\mathbf{x} \circ \mathbf{v}^{(i)}}{\|\mathbf{v}^{(i)}\|_\mathcal{L}^2}
    \end{align}
    Substituting $\alpha$ back into the expression for $\tilde{\mathbf{y}}$, we obtain:
    \begin{equation}
        \tilde{\mathbf{y}} = \mathbf{x} - \frac{\mathbf{x} \circ \mathbf{v}^{(i)}}{\|\mathbf{v}^{(i)}\|_\mathcal{L}^2} \mathbf{v}^{(i)}
    \end{equation}
    To find the point $\mathbf{y}$ on the manifold $\mathbb{L}_\kappa^n$, we normalize $\tilde{\mathbf{y}}$ using the radial projection. First, we compute its squared Lorentz norm. Recalling that $\|\mathbf{x}\|_\mathcal{L}^2 = -1/\kappa$:
    \begin{align}
        \|\tilde{\mathbf{y}}\|_\mathcal{L}^2 &= \left( \mathbf{x} - \frac{\mathbf{x} \circ \mathbf{v}^{(i)}}{\|\mathbf{v}^{(i)}\|_\mathcal{L}^2} \mathbf{v}^{(i)} \right) \circ \left( \mathbf{x} - \frac{\mathbf{x} \circ \mathbf{v}^{(i)}}{\|\mathbf{v}^{(i)}\|_\mathcal{L}^2} \mathbf{v}^{(i)} \right) \\
        &= \mathbf{x} \circ \mathbf{x} - 2 \frac{(\mathbf{x} \circ \mathbf{v}^{(i)})^2}{\|\mathbf{v}^{(i)}\|_\mathcal{L}^2} + \frac{(\mathbf{x} \circ \mathbf{v}^{(i)})^2}{\|\mathbf{v}^{(i)}\|_\mathcal{L}^4} \|\mathbf{v}^{(i)}\|_\mathcal{L}^2 \\
        &= -\frac{1}{\kappa} - \frac{(\mathbf{x} \circ \mathbf{v}^{(i)})^2}{\|\mathbf{v}^{(i)}\|_\mathcal{L}^2} \\
        &= - \frac{1}{\kappa} \left( 1 + \frac{\kappa (\mathbf{x} \circ \mathbf{v}^{(i)})^2}{\|\mathbf{v}^{(i)}\|_\mathcal{L}^2} \right)
    \end{align}

    The distance $d(\mathbf{x}, \mathbf{y})$ is given by the distance formula:
    \begin{equation}
        d(\mathbf{x}, \mathbf{y}) = \frac{1}{\sqrt{\kappa}} \text{arccosh}(-\kappa (\mathbf{x} \circ \mathbf{y}))
    \end{equation}
    Note that $\mathbf{y} = \frac{1}{\sqrt{-\kappa \|\tilde{\mathbf{y}}\|_\mathcal{L}^2}} \tilde{\mathbf{y}}$. We compute the inner product $\mathbf{x} \circ \mathbf{y}$:
    \begin{align}
        \mathbf{x} \circ \mathbf{y} &= \frac{1}{\sqrt{-\kappa \|\tilde{\mathbf{y}}\|_\mathcal{L}^2}} \left( \mathbf{x} \circ \tilde{\mathbf{y}} \right) \\
        &= \frac{1}{\sqrt{-\kappa \|\tilde{\mathbf{y}}\|_\mathcal{L}^2}} \left( \mathbf{x} \circ \mathbf{x} - \frac{(\mathbf{x} \circ \mathbf{v}^{(i)})^2}{\|\mathbf{v}^{(i)}\|_\mathcal{L}^2} \right) \\
         &= \frac{1}{\sqrt{-\kappa \|\tilde{\mathbf{y}}\|_\mathcal{L}^2}} \|\tilde{\mathbf{y}}\|_\mathcal{L}^2 \\
         &= - \frac{1}{\sqrt{\kappa}} \sqrt{-\kappa \|\tilde{\mathbf{y}}\|_\mathcal{L}^2} \\
         &= - \frac{1}{\sqrt{\kappa}} \sqrt{1 + \frac{\kappa (\mathbf{x} \circ \mathbf{v}^{(i)})^2}{\|\mathbf{v}^{(i)}\|_\mathcal{L}^2}}
    \end{align}
    Substituting this back into the distance formula:
    \begin{align}
        d(\mathbf{x}, H_{\mathbf{w}^{(i)}, b_i}) &= \frac{1}{\sqrt{\kappa}} \text{arccosh}\left( \sqrt{1 + \frac{\kappa (\mathbf{x} \circ \mathbf{v}^{(i)})^2}{\|\mathbf{v}^{(i)}\|_\mathcal{L}^2}} \right)
    \end{align}
    Using the identity $\text{arccosh}(\sqrt{1+z^2}) = \text{arcsinh}(|z|)$, we conclude:
    \begin{equation}
        d(\mathbf{x}, H_{\mathbf{w}^{(i)}, b_i}) = \frac{1}{\sqrt{\kappa}} \text{arcsinh}\left( \sqrt{\kappa} \frac{|\mathbf{x} \circ \mathbf{v}^{(i)}|}{\|\mathbf{v}^{(i)}\|_\mathcal{L}} \right)
    \end{equation}
\end{proof}

\begin{restatedproposition}{prop:sign_equivalence}
    Let $\mathbf{p}_i \in \mathbb{L}_\kappa^n$ be the reference point of a hyperplane defined by the normal vector $\mathbf{v}^{(i)} \in T_{\mathbf{p}_i}\mathbb{L}_\kappa^n$. For any point $\mathbf{x} \in \mathbb{L}_\kappa^n$ (where $\mathbf{x} \neq \mathbf{p}_i$), the sign of the Riemannian inner product between the logarithmic map and the normal vector is equivalent to the sign of the Minkowski inner product in the ambient space:
    \begin{equation}
        \text{sign}\left( \langle \text{Log}_{\mathbf{p}_i}(\mathbf{x}), \mathbf{v}^{(i)} \rangle_{\mathbf{p}_i} \right) = \text{sign}(\mathbf{x} \circ \mathbf{v}^{(i)})
    \end{equation}
\end{restatedproposition}

\begin{proof}\label{prop:proof:sign_equivalence}
    Recall that the Riemannian metric $\langle \cdot, \cdot \rangle_{\mathbf{p}_i}$ on the tangent space $T_{\mathbf{p}_i}\mathbb{L}_\kappa^n$ is defined as the restriction of the Minkowski inner product $\circ$. We substitute the explicit formula for the logarithmic map in the Lorentz model:
    \begin{equation}
        \text{Log}_{\mathbf{p}_i}(\mathbf{x}) = \alpha \left( \mathbf{x} + \kappa (\mathbf{p}_i \circ \mathbf{x}) \mathbf{p}_i \right)
    \end{equation}
    where the scalar coefficient is given by $\alpha = \frac{\text{arccosh}(-\kappa (\mathbf{p}_i \circ \mathbf{x}))}{\sqrt{\kappa^2 (\mathbf{p}_i \circ \mathbf{x})^2 - 1}}$. Since $\mathbf{x}$ and $\mathbf{p}_i$ are points on the hyperbolic manifold and $\mathbf{x} \neq \mathbf{p}_i$, the distance between them is strictly positive, implying $\alpha > 0$.
    
    Substituting this into the inner product with the normal vector $\mathbf{v}^{(i)}$:
    \begin{align}
        \langle \text{Log}_{\mathbf{p}_i}(\mathbf{x}), \mathbf{v}^{(i)} \rangle_{\mathbf{p}_i} &= \left( \alpha \left( \mathbf{x} + \kappa (\mathbf{p}_i \circ \mathbf{x}) \mathbf{p}_i \right) \right) \circ \mathbf{v}^{(i)} \\
        &= \alpha \left( (\mathbf{x} \circ \mathbf{v}^{(i)}) + \kappa (\mathbf{p}_i \circ \mathbf{x}) (\mathbf{p}_i \circ \mathbf{v}^{(i)}) \right)
    \end{align}
    Since $\mathbf{v}^{(i)}$ lies in the tangent space at $\mathbf{p}_i$, it satisfies the orthogonality condition $\mathbf{p}_i \circ \mathbf{v}^{(i)} = 0$. The second term in the parentheses vanishes, yielding:
    \begin{equation}
        \langle \text{Log}_{\mathbf{p}_i}(\mathbf{x}), \mathbf{v}^{(i)} \rangle_{\mathbf{p}_i} = \alpha (\mathbf{x} \circ \mathbf{v}^{(i)})
    \end{equation}
    Because $\alpha$ is a strictly positive scalar, it does not affect the sign of the expression. Therefore:
    \begin{equation}
        \text{sign}\left( \alpha (\mathbf{x} \circ \mathbf{v}^{(i)}) \right) = \text{sign}(\mathbf{x} \circ \mathbf{v}^{(i)})
    \end{equation}
\end{proof}

\begin{restatedproposition}{prop:dist_limit_equivalence}
    In the limit where the curvature $\kappa \to 0^+$, the following two expressions for the signed distance tend to the same value:
    \begin{align*}
        d_1 &= \frac{\|\mathbf{w}^{(i)}\|_{\mathcal{L}}}{\sqrt{\kappa}} \arcsinh\left( \sqrt{\kappa} \frac{ \mathbf{x} \circ \mathbf{v}^{(i)}}{\|\mathbf{v}^{(i)}\|_{\mathcal{L}}} \right) \\
        d_2 &= \frac{1}{\sqrt{\kappa}} \arcsinh\left( \sqrt{\kappa} (\mathbf{x} \circ \mathbf{v}^{(i)}) \right)
    \end{align*}
    Specifically, both reduce to the Lorentz product $\mathbf{x} \circ \mathbf{v}^{(i)}$.
\end{restatedproposition}

\begin{proof} 
    \label{prop:proof:dist_limit_equivalence}
    We analyze the limit of both expressions using the Taylor series expansion of the inverse hyperbolic sine function around zero. $\arcsinh(z) = z + \mathcal{O}(z^3)$ for $z \to 0$.
    
    First, we consider $d_1$. As $\kappa \to 0^+$, the argument inside the $\arcsinh$ term approaches zero. Applying the first-order approximation:
    \begin{align}
        \lim_{\kappa \to 0^+} d_1 &= \lim_{\kappa \to 0^+} \frac{\|\mathbf{w}^{(i)}\|_{\mathcal{L}}}{\sqrt{\kappa}} \left( \sqrt{\kappa} \frac{ \mathbf{x} \circ \mathbf{v}^{(i)}}{\|\mathbf{v}^{(i)}\|_{\mathcal{L}}} \right) \\
        &= \|\mathbf{w}^{(i)}\|_{\mathcal{L}} \frac{ \mathbf{x} \circ \mathbf{v}^{(i)}}{\|\mathbf{v}^{(i)}\|_{\mathcal{L}}} \\
        &= \mathbf{x} \circ \mathbf{v}^{(i)}
    \end{align}
    
    Next, we consider $d_2$. Similarly, as $\kappa \to 0^+$, the argument $\sqrt{\kappa} (\mathbf{x} \circ \mathbf{v}^{(i)})$ approaches zero. Applying the approximation:
    \begin{align}
        \lim_{\kappa \to 0^+} d_2 &= \lim_{\kappa \to 0^+} \frac{1}{\sqrt{\kappa}} \left( \sqrt{\kappa} (\mathbf{x} \circ \mathbf{v}^{(i)}) \right) \\
        &= \mathbf{x} \circ \mathbf{v}^{(i)}
    \end{align}
    Both formulations approach to the Minkowski inner product between the input vector and the normal vector in the ambient space, which corresponds to the standard Euclidean dot product in the limit of zero curvature.
\end{proof}

\begin{restatedproposition}{prop:lorentzian_convergence}
    In the limit where the curvature $\kappa \to 0^+$, $h_{\text{Lorentzian},\kappa}$ recovers $h$:
    \begin{align*}
        \lim_{\kappa \to 0^+} h_{\text{Lorentzian},\kappa}(x) = h(x)
    \end{align*}
\end{restatedproposition}
\begin{proof} \label{prop:proof:lorentzian_convergence}
    We aim to show that:
    \begin{equation}
        \lim_{\kappa \to 0^+} \frac{1}{\sqrt{\kappa}} \text{arcsinh}\left(\sqrt{\kappa}h\left(\frac{1}{\sqrt{\kappa}}\sinh{(\sqrt{\kappa}x)}\right)\right) = h(x)
    \end{equation}
    Let $u = \sqrt{\kappa}$. As $\kappa \to 0^+$, it follows that $u \to 0^+$. Substituting $u$ into the expression, we examine the limit:
    \begin{equation}
        L = \lim_{u \to 0^+} \frac{1}{u} \text{arcsinh}\left(u h\left(\frac{1}{u}\sinh{(ux)}\right)\right)
    \end{equation}
    We utilize the first-order Taylor expansion for the hyperbolic sine and its inverse:
    \begin{align}
        \sinh(y) &= y + \mathcal{O}(y^3) \\
        \text{arcsinh}(z) &= z + \mathcal{O}(z^3)
    \end{align}
    First, analyzing the inner term $\frac{1}{u}\sinh(ux)$:
    \begin{equation}
        \frac{1}{u}\sinh(ux) = \frac{1}{u} (ux + \mathcal{O}(u^3)) = x + \mathcal{O}(u^2)
    \end{equation}
    Assuming $h$ is continuous, we have:
    \begin{equation}
        h\left(\frac{1}{u}\sinh(ux)\right) = h(x + \mathcal{O}(u^2)) = h(x) + \mathcal{O}(u^2)
    \end{equation}
    Now, let $Z = u h\left(\frac{1}{u}\sinh{(ux)}\right)$. Substituting the approximation:
    \begin{align}
        Z &= u (h(x) + \mathcal{O}(u^2)) = u h(x) + \mathcal{O}(u^3) \\
         \text{arcsinh}(Z) &= u h(x) + \mathcal{O}(u^3)
    \end{align}
    Finally, substituting this back into the limit $L$:
    \begin{align}
        L &= \lim_{u \to 0^+} \frac{1}{u} (u h(x) + \mathcal{O}(u^3)) \\
        &= \lim_{u \to 0^+} \left( h(x) + \mathcal{O}(u^2) \right) \\
        &= h(x)
    \end{align}
    Thus, the Lorentzian activation function recovers the Euclidean function $h(x)$ when the curvature tends to zero.
\end{proof}

\begin{restatedproposition}{prop:activations_to_output}
    Given $\mathbf{a} \in \mathbb{R}^n$, the vector $\mathbf{y} \in \mathbb{L}^n$ whose distances to the $n$ hyperplanes passing through the origin and orthogonal to the spatial axes are given by $\mathbf{a}$ is:
    \begin{equation}
        \label{eq:output_point}
        \mathbf{y} = 
    \begin{pmatrix}
        y_1 \\
        \overline{\mathbf{y}}
    \end{pmatrix}
    \end{equation}
    such that:
    \begin{align}
        \overline{{\mathbf{y}}} &= \frac{1}{\kappa} \sinh{\left( \sqrt{\kappa} \mathbf{a} \right)} \\
        y_1 &= \sqrt{\frac{1}{\kappa} + \|\overline{\mathbf{y}}\|_E^2} 
    \end{align}
\end{restatedproposition}
\begin{proof} 
    \label{prop:proof:activations_to_output}
    Note that the $i$-th spatial axis is in the $i+1$-th dimension. Let $\mathbf{e}^{(i)}$ in $\mathbb{R}^{n+1}$ such that $[\mathbf{e}^{(i)}]_j = \delta_{ij}$ then:
    \begin{align}
        a_i &= d_\mathbb{L}(\mathbf{x}, H_\mathbf{e^{(i+1)}}) & \text{for }i= 1,...,n \\
        &= \frac{1}{\sqrt{\kappa}}\arcsinh(\sqrt{\kappa} (\mathbf{y} \circ \mathbf{e}^{(i+1)})) \\
        &= \frac{1}{\sqrt{\kappa}}\arcsinh(\sqrt{\kappa} y_{i+1}) \\
        \implies y_{i+1} &= \frac{1}{\sqrt{\kappa}}\sinh(\sqrt{\kappa} a_{i})  
    \end{align}
    And for $\mathbf{y}$ to lie in $\mathbb{L}^n$ 
    \begin{equation}
    y_1 = \sqrt{\frac{1}{\kappa} + \|\overline{\mathbf{y}}\|_E^2} 
    \end{equation}
\end{proof}

\begin{restatedproposition}{prop:solve}
    Our linear layer, defined in \Cref{eq:compute_output}, can learn to embed a tree of depth $h$ with low distortion in $\mathcal{O}(h)$ gradient updates.
\end{restatedproposition}
\begin{proof}
    We will prove this by showing that our linear layer can generate outputs $\mathbf{y}$ whose hyperbolic distance from the origin scales with the number of gradient updates $n$ as $d_\mathcal{H}(\mathbf{o}, \mathbf{y}) = \Theta(n)$.

    We constrain ourselves to only learning $\mathbf{b}$ and keeping $\mathbf{W}$
fixed such that its norm is some constant: $$\|\mathbf{w}^{(j)}\|_E = c\  \forall\  j.$$ 
Without loss of generality, we set $\kappa = 1$. Then, from \cref{eq:defn_v} we get:
\begin{equation}
    \mathbf{v}^{(j)} = \left(-\sinh\left(\frac{b_j}{c}\right), \quad \cosh\left(\frac{b_j}{c}\right) \mathbf{w}^{(j)T}\right).
\end{equation}
It is clear that this scales exponentially with $b_j$, as $\sinh(x) = \Theta(e^x)$, and also $\cosh(x) = \Theta(e^x)$. Moreover, as $\mathbf{v}^{(j)}$ is linearly related to the output coordinate $y_j$, it follows that the output $\mathbf{y}$ scales exponentially with $\mathbf{W}$. 
\end{proof}

\begin{restatedproposition}{prop:cache_inverse}
    Defining $\mathbf{v}^{(i)} = (v^{(i)}_1, \overline{\mathbf{v}}^{(i)})$ as in \cref{eq:defn_v} as a function of $\mathbf{w}^{(i)}$ and $b_i$, the function is invertible for $\|\mathbf{w}^{(i)}\|_E > 0$, and the inverse is given by:
    \begin{align}
        \mathbf{w}^{(i)} &= \frac{\|\mathbf{v}^{(i)}\|_\mathcal{L}}{\|\overline{\mathbf{v}}^{(i)}\|_E} \overline{\mathbf{v}}^{(i)} \\
        b_i &= - \frac{1}{\sqrt{\kappa} \|\mathbf{v}^{(i)}\|_\mathcal{L}} \arcsinh\left( \frac{v^{(i)}_1}{\|\mathbf{v}^{(i)}\|_\mathcal{L}} \right)
    \end{align}
\end{restatedproposition}
\begin{proof}
    \label{prop:proof:cache_inverse} 
    We recall that parallel transport preserves the Lorentzian norm of the transported vector. The weight vector $\mathbf{w}^{(i)}$ is defined in the tangent space at the origin $T_{\mathbf{o}}\mathbb{L}_\kappa^n$, where the Lorentzian norm coincides with the Euclidean norm $\|\mathbf{w}^{(i)}\|_\mathcal{L} = \|\mathbf{w}^{(i)}\|_E$. Consequently, the norm of the cached vector $\mathbf{v}^{(i)} = \Gamma_{\mathbf{o} \to \mathbf{p}_i}(\mathbf{w}^{(i)})$ satisfies:
    \begin{equation}
        \|\mathbf{v}^{(i)}\|_\mathcal{L} = \|\mathbf{w}^{(i)}\|_E
    \end{equation}
    Let $\theta = -\frac{\sqrt{\kappa} b_i}{\|\mathbf{w}^{(i)}\|_E}$. From the definition of the time component in \cref{eq:defn_v}, we have $v^{(i)}_1 = \sinh(\theta) \|\mathbf{w}^{(i)}\|_E$. Substituting the norm equality, we write:
    \begin{equation}
        v^{(i)}_1 = \sinh(\theta) \|\mathbf{v}^{(i)}\|_\mathcal{L} \implies \sinh(\theta) = \frac{v^{(i)}_1}{\|\mathbf{v}^{(i)}\|_\mathcal{L}}
    \end{equation}
    Solving for $b_i$ involves isolating $\theta$:
    \begin{align}
        \theta &= \arcsinh\left( \frac{v^{(i)}_1}{\|\mathbf{v}^{(i)}\|_\mathcal{L}} \right) \\
        -\frac{\sqrt{\kappa} b_i}{\|\mathbf{v}^{(i)}\|_\mathcal{L}} &= \arcsinh\left( \frac{v^{(i)}_1}{\|\mathbf{v}^{(i)}\|_\mathcal{L}} \right) \\
        b_i &= - \frac{\|\mathbf{v}^{(i)}\|_\mathcal{L}}{\sqrt{\kappa}} \arcsinh\left( \frac{v^{(i)}_1}{\|\mathbf{v}^{(i)}\|_\mathcal{L}} \right)
    \end{align}
    To recover $\mathbf{w}^{(i)}$, we look at the spatial component $\overline{\mathbf{v}}^{(i)} = \cosh(\theta) \mathbf{w}^{(i)}$. We compute $\cosh(\theta)$ using the identity $\cosh(x) = \sqrt{1+\sinh^2(x)}$:
    \begin{align}
        \cosh(\theta) &= \sqrt{1 + \left(\frac{v^{(i)}_1}{\|\mathbf{v}^{(i)}\|_\mathcal{L}}\right)^2} = \frac{\sqrt{\|\mathbf{v}^{(i)}\|_\mathcal{L}^2 + (v^{(i)}_1)^2}}{\|\mathbf{v}^{(i)}\|_\mathcal{L}}
    \end{align}
    Given that $\mathbf{v}^{(i)}$ is a spacelike vector, $\|\mathbf{v}^{(i)}\|_\mathcal{L}^2 = -(v^{(i)}_1)^2 + \|\overline{\mathbf{v}}^{(i)}\|_E^2$. Substituting this into the numerator:
    \begin{equation}
        \cosh(\theta) = \frac{\sqrt{-(v^{(i)}_1)^2 + \|\overline{\mathbf{v}}^{(i)}\|_E^2 + (v^{(i)}_1)^2}}{\|\mathbf{v}^{(i)}\|_\mathcal{L}} = \frac{\|\overline{\mathbf{v}}^{(i)}\|_E}{\|\mathbf{v}^{(i)}\|_\mathcal{L}}
    \end{equation}
    Finally, we invert the spatial equation:
    \begin{equation}
        \mathbf{w}^{(i)} = \frac{1}{\cosh(\theta)} \overline{\mathbf{v}}^{(i)} = \frac{\|\mathbf{v}^{(i)}\|_\mathcal{L}}{\|\overline{\mathbf{v}}^{(i)}\|_E} \overline{\mathbf{v}}^{(i)}
    \end{equation}
\end{proof}

\end{document}